\documentclass[lettersize,journal]{IEEEtran}
\usepackage{amsmath,amsfonts}
\usepackage{algorithmic}
\usepackage{algorithm}
\usepackage{array}
\usepackage[caption=false,font=footnotesize,labelfont=rm,textfont=rm]{subfig}
\usepackage{textcomp}
\usepackage{stfloats}
\usepackage{url}
\usepackage{verbatim}
\usepackage{graphicx}
\usepackage[numbers]{natbib}

\usepackage{amssymb}
\usepackage{amsmath}
\usepackage{amsthm}
\usepackage{url}
\usepackage{multirow}
\usepackage{multicol}
\usepackage{makecell}
\usepackage{booktabs}       
\usepackage{hyperref}

\newtheorem{theorem}{Theorem}[]
\newtheorem{corollary}{Corollary}[]
\newtheorem{definition}{Definition}[]

\hypersetup{colorlinks=true,
linkcolor=cyan,citecolor=cyan,urlcolor=cyan
}

\begin{document}

\title{Variational Graph Generator for Multi-View Graph Clustering}

\author{
    Jianpeng~Chen,
    Yawen~Ling,
    Jie~Xu,
    Yazhou~Ren,~\IEEEmembership{Member,~IEEE,}
    Shudong~Huang,
    Xiaorong~Pu,
    Zhifeng~Hao,~\IEEEmembership{Senior Member,~IEEE,}
    Philip~S.~Yu,~\IEEEmembership{Fellow,~IEEE,}
    Lifang~He,~\IEEEmembership{Senior Member,~IEEE}

\thanks{(Corresponding author: Yazhou Ren)}
\thanks{Jianpeng Chen was with the School of Computer Science and Engineering, University of Electronic Science and Technology of China, Chengdu 611731, China. He is now with the Department of Computer Science, Virginia Tech, Blacksburg, VA 24061, USA (e-mail: jianpengc@vt.edu).}
\thanks{Yawen Ling, Jie Xu are with the School of Computer Science and Engineering, University of Electronic Science and Technology of China, Chengdu 611731, China (e-mail: \{yawen.Ling, jiexuwork\}@outlook.com).}
\thanks{Yazhou Ren and Xiaorong Pu are with the School of Computer Science and Engineering, University of Electronic Science and Technology of China, Chengdu 611731, China, and also with the Shenzhen Institute for Advanced
 Study, University of Electronic Science and Technology of China, Shenzhen 518000, China (e-mail: \{yazhou.ren, puxiaor\}@uestc.edu.cn).}
\thanks{Shudong Huang is with the College of Computer Science, Sichuan University, Chengdu 610065, China (e-mail: huangsd@scu.edu.cn).}
\thanks{Zhifeng Hao is with the Department of Mathematics, the College of Science, Shantou University, Shantou 515063, China (e-mail: haozhifeng@stu.edu.cn).}
\thanks{Philip S. Yu is with the Department of Computer Science, University of Illinois Chicago, Chicago, IL 60607, USA (e-mail: psyu@cs.uic.edu).}
\thanks{Lifang He is with the Department of Computer Science and Engineering, Lehigh University, PA 18015, USA (e-mail: lih319@lehigh.edu).}
}

\markboth{IEEE Transactions on Neural Networks and Learning Systems, VOL. XX, NO. X, DECEMBER 2022}%
{Shell \MakeLowercase{\textit{et al.}}: A Sample Article Using IEEEtran.cls for IEEE Journals}

\IEEEpubid{0000--0000/00\$00.00~\copyright~2021 IEEE}

\maketitle

\begin{abstract}
  Multi-view graph clustering (MGC) methods are increasingly being studied due to the explosion of multi-view data with graph structural information. The critical point of MGC is to better utilize view-specific and view-common information in features and graphs of multiple views. However, existing works have an inherent limitation that they are unable to concurrently utilize the consensus graph information across multiple graphs and the view-specific feature information. To address this issue, we propose Variational Graph Generator for Multi-View Graph Clustering (VGMGC). Specifically, a novel variational graph generator is proposed to extract common information among multiple graphs. This generator infers a reliable variational consensus graph based on a priori assumption over multiple graphs. Then a simple yet effective graph encoder in conjunction with the multi-view clustering objective is presented to learn the desired graph embeddings for clustering, which embeds the inferred view-common graph and view-specific graphs together with features. Finally, theoretical results illustrate the rationality of the VGMGC by analyzing the uncertainty of the inferred consensus graph with the information bottleneck principle. 
  Extensive experiments demonstrate the superior performance of our VGMGC over SOTAs. The source code is publicly available at \url{https://github.com/cjpcool/VGMGC}.
\end{abstract}

\section{Introduction}
\label{secIntro}
\IEEEPARstart{A}{ttributed} graph clustering has received increasing attention with the development of graph structure data, such as social networks, academic networks, and world wide web~\cite{AGC,DGI,SGC}. Attributed graph clustering aims to better analyze the graphs with its attributed features for clustering task, which can benefit lots of fields such as product recommendation~\cite{Kakarash2022ATA} and community detection~\cite{9511798}. On the other hand, many graph structure data contain multiple views in real-world applications. For example, graphs in social networks can be constructed based on common interests or followers of users, and graphs in academic networks can be established by co-authors or citations of papers. Therefore, lots of multi-view graph clustering (MGC) methods have emerged in recent years for better mining the information of multi-view graphs~\cite{MAGCN, o2multi,MVGCIC,9842356}.

\IEEEpubidadjcol
\begin{figure}[t]
    \centering
    \includegraphics[width=0.9\linewidth]{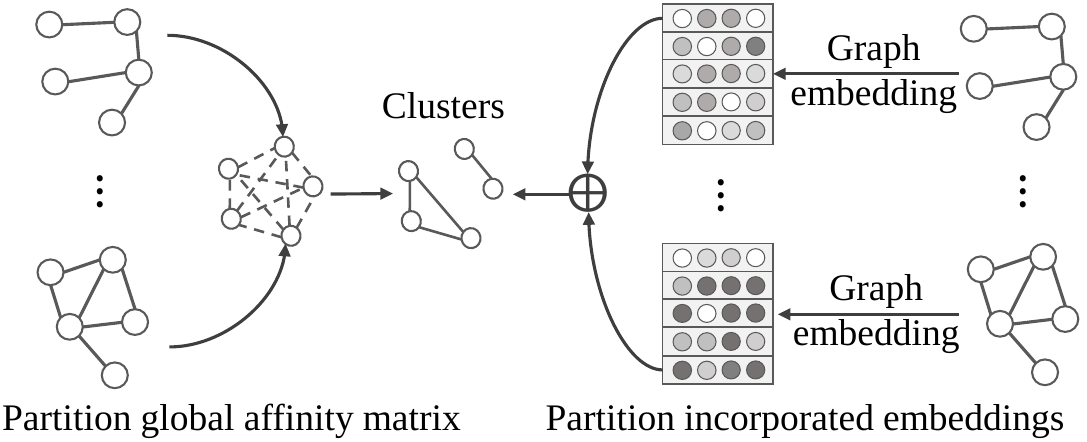}
    \caption{Two types of multi-view graph clustering. $\oplus$ denotes any combination operation (\emph{e.g.}, concatenation). One type tries to capture consensus topological information, but view-specific information in features is ignored (left); another type embeds view-specific graphs at first, then learns a global embedding cooperatively, but global topological information is lost (right).}
    \label{fig:2categories}
\end{figure}

Generally, MGC methods could be categorized into two types~\cite{o2multi,nie2017self,MVCGR,9509355} as shown in Figure~\ref{fig:2categories}. One type generates a global affinity matrix of all samples for all views and then groups them into different clusters~\cite{nie2017self,xia14RMSC,GMC,RGIMVC}. These methods successfully demonstrate the effectiveness of learning a consensus graph for analyzing all views. Another type learns representations for individual views via graph embedding techniques and then incorporates them into global representations. Subsequently, traditional clustering methods, such as $k$-means~\cite{kmeans}, can be adopted to obtain clustering results~\cite{MAGCN,R2FGC,9516695}. Benefitting from the graph embedding ability of Graph Neural Networks (GNNs)~\cite{SGC,GCN,GAT,S2GC}, the two types of methods have been greatly improved~\cite{MAGCN,o2multi,9516695,pan2021multi,lin2021graph}. 
Nevertheless, both studies have inherent limitations. For the first type of work, the consensus graph reflects the common topological information of all views. But if we directly partition the consensus graph, the latent view-specific information in features will be ignored. As for the second type of work, each view's graph embeddings are learned separately, so the global topological information might be ignored. Therefore, it is desirable to develop a method that is able to assimilate both the view-common and view-specific topological information, and in the meantime, embed the feature information by utilizing the superior graph embedding ability of GNNs.


To achieve the above goal, we will have two major challenges: 1) how to explore the consensus information to generate a consensus graph that is applicable to GNNs, and 2) how to design a GNN that could embed the view-specific information and the generated consensus graph into features together.
\textbf{C1. Consensus exploration for graph generation:}
Considering the first challenge, the generated consensus graph requires some properties. First, it must imply the \emph{common information} from both features and graphs. Second, if there is an observed graph that is unreliable, it should depend less on this unreliable graph, in other words, different graphs have different probabilities to be employed for generating the consensus graph, requiring us to infer a consensus graph \emph{with a specific probability}.  Third, it should be \emph{sparse} to avoid excessive storage and computational consumption, which can simultaneously enhance robustness.~\cite{robustSparse,2020ChriskWL}. 
Based on these concerns, the technique of variational inference is a perfect match, which could infer sparse and discrete variables with a specific probability. More importantly, with our novel design, the proposed variational graph generator can extract the common information among multiple graphs and global features, as well as some task-relevant information (see Section~\ref{secVGG}).
\textbf{C2. GNN designing for consensus and specific information fusion:} Regarding the second challenge, we design a simple yet effective encoder with a parameter-free message-passing method to embed the view-specific and view-common graph information together (see Section~\ref{secGraphEnc}). The design of this graph encoder comes mainly from previous related research on GNNs~\cite{SGC,2020ChriskWL,GIN} which has suggested that many parameters in GNNs are redundant, and the powerful representation ability of GNNs comes from the message-passing process.
In addition, we introduce our multi-view fusion method to incorporate graph embeddings from all views (see Section~\ref{SecMVFusion}). 
For brevity, we call the proposed framework VGMGC (Variational Graph Generator for Multi-View Graph Clustering).

Furthermore, we have theoretically analyzed the rationality of the generated graph (see Section~\ref{secTheoAna}), and empirically demonstrated the effectiveness of VGMGC (see Section~\ref{secExp}).
In summary, the main contributions of this work include:
\begin{itemize}
    \item \textbf{Novel variational graph generator:} We propose a novel variational graph generator based on a priori assumption over multi-graphs to produce a variational consensus graph with desired properties for GNNs. To the best of our knowledge, this work is the first trial utilizing variational inference to infer a discrete graph for multi-view graph learning. The use of variational inference for inferring discrete graphs in the context of multi-view graph learning is a significant advancement.
    \item \textbf{Efficient parameter-free message-passing encoder:} We present an efficient encoder that embeds the consensus graph and view-specific graphs into embeddings in a parameter-free message-passing manner. The design of a parameter-free encoder advances the computational efficiency and scalability of multi-graph processing.
    \item \textbf{Theoretical and empirical validation:} Theoretical analysis illustrates the rationality of the inferred consensus graph and our proposals, providing a reference for research about graph information bottleneck. Experiments on eight diverse datasets demonstrate the superior performance of our method compared to SOTAs.
\end{itemize}

\section{Related Works}
\subsection{Graph Neural Networks for Multi-View Graph Clustering}
\citet{Ren2022DeepCA} reviewed some methods for deep multi-view clustering, such as dual self-paced multi-view clustering~\cite{HUANG2021184}, self-supervised discriminative feature learning for deep multi-view clustering~\cite{SDMVC,huang2023self,huang2021non,xu2023untie}, and unsupervised deep embedding for cluster analysis~\cite{DEC,BMSC}, but the issues in multi-view graph clustering still remain to be explored.
\citet{GIN} has theoretically proved that GNNs can efficiently embed graph structural information with features. Promoted by the success of GNNs, several multi-view clustering methods have emerged. Specifically, \citet{o2multi} first employed GNN~\cite{GCN} for MGC. It selects an informative graph with features and conducts GNN to generate embeddings for clustering and multi-view information is injected by backpropagation. However, the information implied in the unselected graphs is largely lost in this way. \citet{MAGCN} adopted GNN for multiple features with a shared graph, but it cannot be generalized to data with multiple graphs. In contrast, \citet{SAMGC} designed a GNN for view-shared features with multiple graphs named SAMGC. Alternatively, researchers in~\cite{pan2021multi} and~\cite{lin2021graph} embedded the graph-structured information in each view separately through the graph filter and then learned a consensus graph for spectral clustering. However, due to the fact that their clustering results are only determined by the generated graph, some useful information within features is ignored.

Most recently, some innovative approaches in multi-view graph clustering have emerged. Notably, the dual label-guided graph refinement for multi-view graph clustering (DuaLGR)~\cite{dulgr} and the method for reconstructing structure in graph-agnostic clustering (DGCN)~\cite{pan23b} have been introduced, specifically targeting heterophilous graphs in multi-view data scenarios. Following the two works, \citet{AGHGC} proposed an adaptive hybrid graph filter (AGHGC) for heterophilous graph multi-view clustering. \citet{R2FGC} have made significant strides by focusing on the redundancy of edges within graphs, leading to the development of a relational redundancy-free graph learning method. Additionally, \citet{SDSNE} proposed the use of stationary diffusion as a novel technique to acquire a consensus graph. Similarly \citet{pgsc} advanced the field by devising a method to obtain the consensus graph through the identification of optimal neighbors. However, these two methodologies fundamentally adhere to the first type mentioned before. Consequently, they are subject to similar inherent limitations, where the clustering results are only determined by the consensus graph, thus leading to a notable loss of feature information.

In contrast, the proposed VGMGC could reserve useful feature information and utilize the powerful graph embedding ability of GNNs to assimilate both the view-consensus and view-specific graph information for clustering tasks.

\subsection{Variational Inference for Multi-View Graph Clustering}
\citet{vae} proposed variational autoencoders (VAE), which have subsequently influenced many fields. For example, \citet{vae4clust} and \citet{cui2024novel} introduced VAE for clustering task. \citet{multiVAE} disentangled view-specific and view-common information via VAE for multi-view clustering task. In the field of graph clustering, \citet{GAE} first adopted this technique to graph representation learning for link prediction. Following this work, \citet{ARVGAE} proposed a variational adversarial approach to learn a robust graph embedding. \citet{TGNN} extended graph VAE to infer stochastic node representations. \citet{D-VGAE} proposed D-VGAE for link prediction by considering node similarity and node popularity. However, all these works aim at inferring the node representations rather than directly inferring a graph structure.

It seems hard to utilize variation to infer a graph since the observed input graph is certain and the priori distribution of the input graph is unknown. Even so, a few works have tried to infer a graph. \citet{pgexplainer} introduced a prior to graph, and sampled a new graph from observed one via uniform distribution. However, this graph is not essentially a generated one, but only a subgraph of the observed graph. \citet{advercialGCN} introduced two hyper-parameters to control the degree of uncertainty of a graph and generated a graph based on variational inference and $k$-nearest neighbors (kNN), but it also has two potential issues. First, the selection of the two hyper-parameters is groundless. Second, its variational ability is greatly restricted by kNN, and the $k$ in kNN is difficult to set.  In the most recent years, some works explored the graph generation ability of diffusion models for high-quality discrete graph generation~\cite{chen2023efficient,vignac2022digress}. However, these works require large numbers of observed graphs for supervision, and the diffusion-denoising process is time-consuming, which limits their application in the MGC task.
On the contrary, in the proposed variational graph generator, the uncertainty for inferring the consensus graph depends on how reliable the observed graphs are for the clustering task (Theorems~\ref{theorem1}~and~\ref{theorem2}). Therefore, the inferred graph is more robust and task-related. Moreover, to the best of our knowledge, this is the first trial to directly infer a graph for multi-view clustering.
\section{Methods}
\subsection{Problem Statement}
Given multiple graphs with their attributed features $\{\mathcal{G}^v=(\mathbf{X}^v, \mathbf{A}^v)\}_{v=1}^V$, where $\mathbf{X}^v \in \mathbb{R}^{n \times d_v}$ denotes the features of $n$ nodes in the graph of the $v$-th view, and $\mathbf{A}^v=\{a_{ij} \mid a_{ij} \in \{0,1\}\}$ denotes the corresponding adjacency matrix (with self-loop), respectively, MGC aims to partition the $n$ samples into $c$ clusters. 
Let a diagonal matrix $\mathbf{D}^v$ (where $\mathbf{D}^v_{ii}=\sum_ja^v_{ij}$) represent the degree matrix of $\mathbf{A}^v$, the normalized adjacent matrix is defined as $\mathbf{\widetilde{A}}^v=(\mathbf{D}^v)^{-1}\mathbf{A}^v$. 
We define $\mathbf{\overline{X}} \in \mathbb{R}^{n \times d'}$ as a global feature shared by all views. Generally, if the dataset does not contain this feature, it is generated by concatenating $\{\mathbf{X}^v\}^V_{v=1}$, \emph{i.e.}, $\mathbf{\overline{X}}=Concat(\mathbf{X}^{\it1}, \mathbf{X}^{\it2}, ..., \mathbf{X}^V)$ where $Concat(\cdot,\cdot)$ denotes concatenating operation. 
Specifically, let $\mathbf{\overline{Z}}\in \mathbb{R}^{n \times (V\cdot D)}$ and $\mathbf{Z}^v \in \mathbb{R}^{n \times D}$ be the global and each view's latent representation, respectively. The goal of this work is to learn better MGC by leveraging the novel variational graph generator and an effective graph encoder. The overall descriptions of notations are shown in Appendix~\ref{appNotation}.

\subsection{Overview of Variational Graph Generator for Multi-View Graph Clustering}
\begin{figure*}[t]
  \centering
  \includegraphics[width=1\linewidth]{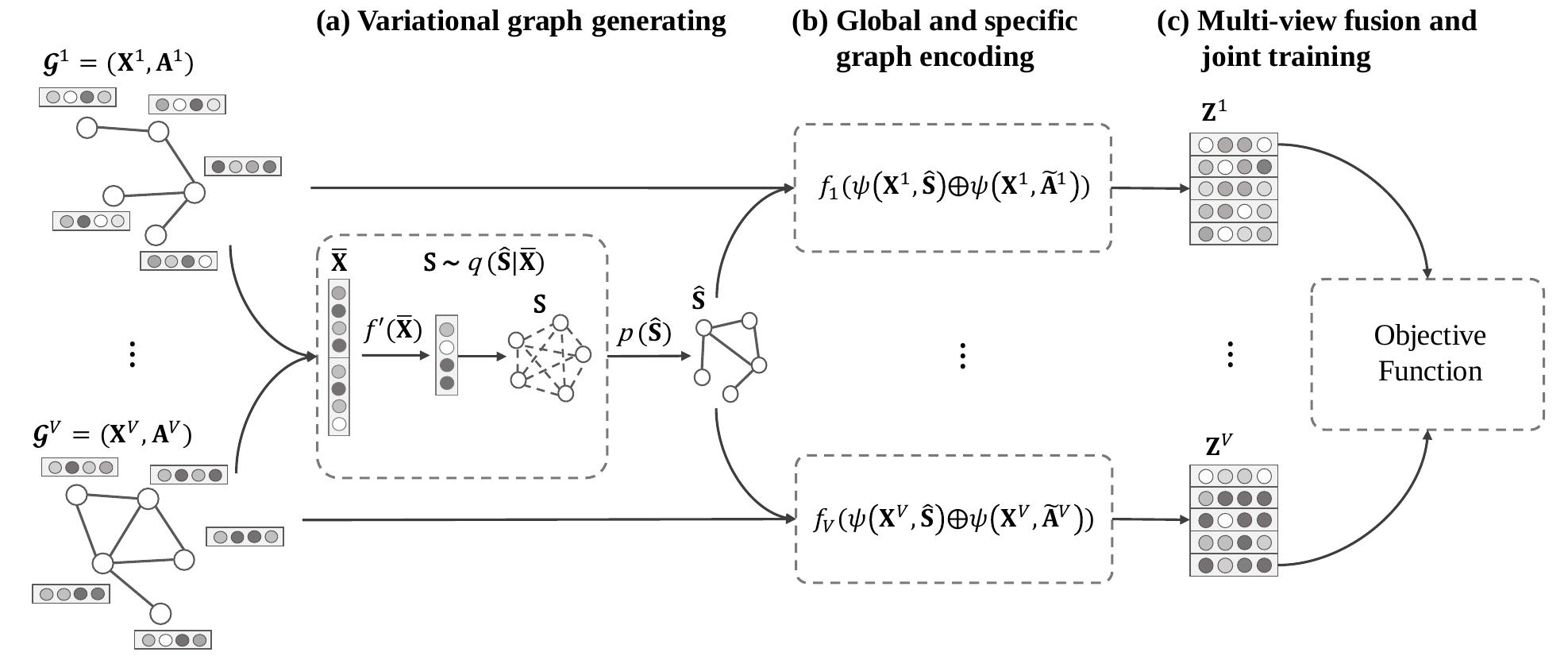}
  \caption{Overview of the VGMGC framework. The symbols $\overline{\mathbf{X}}$, $\hat{\mathbf{S}}$, $\mathbf{A}^v$ and $\mathbf{Z}^v$ represent the global representation derived from all views $\{\textbf{X}^v\}^V_{v=1}$, the learned consensus graph, the adjacency matrix, and final output representations of $v$-th view, respectively. $q(\cdot)$ and $p(\cdot)$ denote the posterior and prior probabilities, respectively. Symbol $\oplus$ denotes the concatenation operation. $\psi(\cdot)$ is the proposed message passing scheme in Eq.~\eqref{eqMessagePassing} and $f(\cdot)$ is the MLP for low-dimensional representation learning. VGMGC first infers a cross-view consensus graph $\mathbf{\hat{S}}$ by the variational graph generator (a). Then this variational consensus graph, combined with each attributed view-specific graph, is processed by the global and specific graph encoder (b). This encoder generates the latent features $\mathbf{Z}^v$ for each view. Finally, the latent features from all views are weighted and concatenated for training and clustering (c).}
  \label{overview}
\end{figure*}
The overview of VGMGC is shown in Figure~\ref{overview}. In the following, we will concentrate on three questions to disassemble our proposed framework: a) how the variational consensus graph is inferred from multi-view graph data (see Section~\ref{secVGG}), b) how the variational consensus graph is incorporated with each view's specific graph and features to generate latent feature $\mathbf{Z}^v$ (see Section~\ref{secGraphEnc}), and c) how to fuse multi-view latent features for collaborative training and clustering (see Section~\ref{SecMVFusion}).

\subsection{Variational Graph Generator}
\label{secVGG}
We first present our variational graph generator, which can infer a discrete consensus graph containing the common feature information as well as common topological information simultaneously, and we will introduce how the two kinds of common information are captured in each step. On top of that, another difference between our variational graph generator and previous VAEs~\cite{vae,vae4clust,multiVAE,GAE} is that we aim at inferring a discrete graph directly, but previous VAEs are to infer latent embeddings of nodes and construct the graph by computing the similarities between these embeddings.

Let $\mathbf{\hat{S}}=\{\hat{s}_{ij} \mid \hat{s}_{ij} \in \{0,1\}\}$ be an independent consensus graph. We aim to model the multi-graph data $\{\mathbf{A}^v\}^V_{v=1}$ to find a better consensus graph, \emph{i.e.}, $\mathbf{\hat{S}}$. Therefore, this generative model tries to learn the joint distribution:
\begin{equation}
\begin{aligned}
    p(\mathbf{X}^v, \mathbf{\hat{S}}, \mathbf{A}^v) &=p(\mathbf{A}^v \mid \mathbf{X}^v,\mathbf{\hat{S}})p(\mathbf{X}^v,\mathbf{\hat{S}}), 
\end{aligned}
\end{equation}
where $p(\mathbf{X}^v, \mathbf{\hat{S}}, \mathbf{A}^v)$ represents the joint probability of $\mathbf{X}^v$, $\mathbf{\hat{S}}$, and $\mathbf{A}^v$.
\subsubsection{A Priori Assumption Over Multi-Graph}
\label{secAssPrior}
The priori assumption could not only give the consensus graph a definition but also inject some task-relevant information and consensus information among graphs into the consensus graph.

Begin with the definition of this graph, we consider that the consensus graph is a Gilbert random graph~\cite{GilbertRandomGraph}, where the edges in $\mathbf{\hat{S}}$ are conditionally independent with each other, so the probability of generating $\mathbf{\hat{S}}$ could be factorized as:
\begin{equation}
\nonumber
    p(\mathbf{\hat{S}}) = \prod_{i,j}p(\hat{s}_{ij}).
\end{equation}1qq
An evident instantiation of $p(\hat{s}_{ij})$ is the Bernoulli distribution, \emph{i.e.}, $\hat{s}_{ij} \thicksim Bern(\beta_{ij})$. $p(\hat{s}_{ij}=1)=\beta_{ij}$ is the prior probability of whether there exists an edge between nodes $i$ and $j$, where $\boldsymbol{\beta}$ are parameters determined by all observed graphs $\{\mathbf{A}^v\}_{v=1}^V$ with their corresponding optimizable beliefs $\{b^v\}_{v=1}^V$. 
The belief $b^v \in (0,1]$ of $\mathbf{A}^v$, which could be obtained by a self-supervised method (which will be introduced in Section~\ref{SecMVFusion}), is to control how much the $v$-th graph should be trusted for our clustering task. Then, we aggregate all graphs $\{\mathbf{A}^v\}_{v=1}^V$ with their corresponding beliefs $\{b^v\}_{v=1}^V$ by computing the consensus among them to learn the prior parameter $\boldsymbol{\beta}$:
\begin{equation}
\label{eqBeta}
    \beta_{ij} = \frac{\sum_{v} [b^v \cdot a^v_{ij} + (1-b^v) \cdot (1-a^v_{ij})]}{\sum_{v}b^v}.
\end{equation}

From Eq.~\eqref{eqBeta}, the belief (or task relevance $\boldsymbol{\beta}$) of each graph and the consensus among the multiple graphs can be injected into the consensus graph. This conclusion is also proved in Theorems \ref{theorem1} and \ref{theorem2}.

\subsubsection{Inference Process} 
In the inference process, the variational graph generator tries to capture the common feature information from global features by computing the posterior probability.

Specifically, the variational consensus graph $\mathbf{\hat{S}}$ is generated from global features $\mathbf{\overline{X}}$, so the posterior of $\mathbf{\hat{S}}$ could be written as $p(\mathbf{\hat{S}} \mid \mathbf{\overline{X}})$. However, considering the posterior is intractable to be calculated~\cite{vae, vae4clust}, we propose a neural network which is based on self-attention~\cite{Transformer} to approximate it, \emph{i.e.}, $q_{\phi'}(\mathbf{\hat{S}} \mid \mathbf{\overline{X}})$ where $\phi'$ denotes the trainable parameters of the neural network. The proposed neural network could be written as:
\begin{equation}
\label{eqAlpha}
    \boldsymbol{\alpha} = \mathbf{K}\mathbf{Q}^T,
    \text{ where }\mathbf{K} = f'(\mathbf{\overline{X}}), \mathbf{Q} = \mathbf{K} \mathbf{W},
\end{equation}
where $\boldsymbol{\alpha}=\{\alpha_{ij}\} \in \mathbb{R}^{n \times n}$ denotes the neurons used to obtain variable $\mathbf{\hat{S}}$, $f'$ represents a multilayer perceptron (MLP), and $\mathbf{W}$ is the trainable parameters. Therefore, $\phi'$ contains trainable parameters in $f'$ and $\mathbf{W}$. 

Moreover, due to the structural information $\boldsymbol{\alpha}$ needs to reserve the important information of global features $\mathbf{\overline{X}}$, we are expected to maximize the mutual information between $\boldsymbol{\alpha}$ and $\mathbf{\overline{X}}$, \emph{i.e.}, $I(\boldsymbol{\alpha},\mathbf{\overline{X}})$.
However, maximizing $I(\boldsymbol{\alpha},\mathbf{\overline{X}})$ is intractable, we achieve this by maximizing its lower bound.

In doing so, from Eq.~\eqref{eqAlpha}, we can know that the information in $\boldsymbol{\alpha}$ is deterministic given $\mathbf{K}$ and $\mathbf{Q}$. So, we have:
\begin{equation}
\nonumber
   \arg\max I(\boldsymbol{\alpha}, \mathbf{\overline{X}})  \cong \arg\max I(\mathbf{K},\mathbf{\overline{X}}) \text{ and} \arg\max I(\mathbf{Q}, \mathbf{\overline{X}}).
\end{equation}
Considering the three variables $\mathbf{K}$, $\mathbf{Q}$ and $\mathbf{\overline{X}}$, we have the Markov chain: $\mathbf{\overline{X}} \rightarrow \mathbf{K} \rightarrow \mathbf{Q}$. According to Data Processing Inequality, we can get 
$    I(\mathbf{K},\mathbf{\overline{X}}) \geq I(\mathbf{Q}, \mathbf{\overline{X}}).
$
This demonstrates that $I(\mathbf{Q}, \mathbf{\overline{X}})$ is the lower bound of $I(\mathbf{K},\mathbf{\overline{X}})$, and we have:
\begin{equation}
\nonumber
    \arg\max I(\mathbf{K},\mathbf{\overline{X}}) \text{ and } \arg\max I(\mathbf{Q}, \mathbf{\overline{X}}) \cong \arg\max I(\mathbf{Q}, \mathbf{\overline{X}}).
\end{equation}
By concluding the two equations above, we obtain Eq.~\eqref{eqMIalpha}:
\begin{equation}
\label{eqMIalpha}
    \underset{\phi'}{\arg\max} I(\boldsymbol{\alpha},\mathbf{\overline{X}}) \cong \underset{\phi'}{\arg\max} I(\mathbf{Q},\mathbf{\overline{X}}).
\end{equation}

Eq.~\eqref{eqMIalpha} enforces the posterior probability to reserve as much global feature information as possible.

On the other hand, due to the discrete nature of graph $\mathbf{\hat{S}}$, we adopt reparameterization trick to optimize these parameters with gradient-based methods~\cite{Gumbelsoft}. In this work we utilize binary concrete distribution for relaxation, \emph{i.e.}, $s_{ij} \thicksim BinConcrete(\alpha_{ij}, \tau)$, where $\tau > 0$ and $\alpha_{ij}$ represent the temperature and location parameters respectively. Specifically, let variable $U \thicksim Uniform(0,1)$, the relaxed weight $s_{ij} \in (0,1)$ of edge $(i,j)$ is calculated as:

\begin{equation}
\begin{aligned}
\label{eqRepara}
    q_{\phi'}(\hat{s}_{ij} \mid \mathbf{\overline{x}}) 
    =\sigma((\log{U} - \log{(1-U)} + \alpha_{ij})/\tau),\\
    \text{ with }U \thicksim Uniform(0,1),
\end{aligned}
\end{equation}
where $\sigma(\cdot)$ denotes the \emph{Sigmoid} function. It is easy to demonstrate the rationality of using the binary concrete distribution of $s_{ij}$ to approximate the priori Bernoulli distribution of $\hat{s}_{ij}$~\cite{binConc}.

\subsubsection{Generative Process}
The generative process can enhance the common topological information among multiple graphs by reconstructing every graph.
Concretely, after obtaining the latent consensus graph $\mathbf{\hat{S}}$, in the generative process, we model all views' graphs $\{\mathbf{A}^v\}^V_{v=1}$ in a collaborative manner. Let $\{\xi^v\}^V_{v=1}$ be the learnable parameters of decoder, the likelihood of $v$-th graph is as follow:
\begin{equation}
\label{eqGen}
    \mathbf{\breve{A}}^v = p_{\xi^v}(\mathbf{A}^v \mid \mathbf{X}^v, \mathbf{\hat{S}}).
\end{equation}
\subsubsection{Evidence Lower Bound}
 According to Jensen's inequality and our Bayesian net (Figure~\ref{fig:BayesianNet}), it is easy to prove that maximizing the likelihood of observed graphs (\emph{i.e.}, $p(\mathbf{A}^v)$) is equal to maximizing the evidence lower bound (ELBO):
\begin{equation}
\nonumber
\begin{aligned}
    &\sum^V_{v=1} \log{p(\mathbf{A}^v)} 
    = \sum^V_{v=1} \log{\int \int p({a}^v, {\hat{s}}, {x}^v) d{\hat{s}} d{x}^v}\\
    &=\sum^V_{v=1} \log \int\int \frac{q({\hat{s}}\mid  \overline{x})}{q({\hat{s}}\mid \overline{x})} p({a}^v, {\hat{s}}, {x}^v) d{\hat{s}} d{x}^v\\
    &= \sum^V_{v=1} \log \int \mathbb{E}_{q({\hat{s}}\mid \overline{x})} \frac{ p({a}^v, {\hat{s}}, {x}^v)}{q({\hat{s}}\mid \overline{x})} dx^v\\
    &\geq \sum^V_{v=1} \mathbb{E}_{q({\hat{s}}\mid \overline{x})} \log \int \frac{ p({a}^v, {\hat{s}}, {x}^v)}{q({\hat{s}}\mid \overline{x})} dx^v, \text{ (Jensen's inequality),}\\
    &= \sum^V_{v=1} \mathbb{E}_{q({\hat{s}}\mid \overline{x})}\log \int \frac{p(a^v\mid \hat{s}, x^v)p(\hat{s})}{q(\hat{s}\mid \overline{x})} p(x^v) dx^v, \text{ ($\hat{s}\perp\!\!\!\!\perp x^v \mid \emptyset$),}\\
    &\geq\sum^V_{v=1} \mathbb{E}_{q({\hat{s}}\mid \overline{x})p(x^v)}\log p(a^v\mid \hat{s},x^v) + \mathbb{E}_{q({\hat{s}}\mid \overline{x})}\log\frac{p(\hat{s})}{q({\hat{s}}\mid \overline{x})}\\
    &=\sum^V_{v=1} \mathbb{E}_{q({\hat{s}}\mid \overline{x})p(x^v)}\log p(a^v\mid \hat{s},x^v) - KL(q({\hat{s}}\mid \overline{x}) \Vert p(\hat{s})).
\end{aligned}
\end{equation}
Here, $\hat{s}\perp\!\!\!\!\perp x^v \mid \emptyset$ is derived from the Bayesian net (Figure~\ref{fig:BayesianNet}).
\begin{figure}[t]
    \centering
    \includegraphics[width=0.6\linewidth]{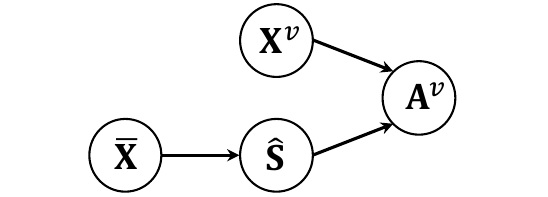}
    \caption{The Bayesian net of our variational graph generator about variables: $\mathbf{A}^v$, $\mathbf{\hat{S}}$, $\mathbf{\overline{X}}$ and $\mathbf{X}^v$.}
    \label{fig:BayesianNet}
\end{figure}
Therefore, our first objective function, \emph{i.e.}, ELBO, can be written as:
\begin{equation}
\begin{aligned}
\label{eqELBO}
    \underset{\phi',\{\xi^v\}^V_{v=1}}{\arg\max} \mathcal{L}_{E} &= \sum^V_{v=1}\mathbb{E}_{q_{\phi'}(\mathbf{\hat{S}} \mid {\mathbf{\overline{X}}})p(\mathbf{X}^v)}\left[\log{p_{\xi^v}(\mathbf{A}^v \mid \mathbf{\hat{S}},\mathbf{X}^v)}\right] \\
    &- KL\left(q_{\phi'}(\mathbf{\hat{S}} \mid \mathbf{\overline{X}}) \parallel p(\mathbf{\hat{S}}) \right).
\end{aligned}
\end{equation}

The first term of Eq.~\eqref{eqELBO} is the expectation of the conditional likelihood over the approximate posterior, \emph{i.e.}, expected log likelihood (ELL), and the second term is the KL divergence between the approximate posterior of the variational consensus graph and its prior. Under our priori assumption, \emph{i.e.}, $p(s_{ij})=Bern(\beta_{ij})$, the KL divergence is bounded:
\begin{equation}
\begin{aligned} 
\label{eqELBOKL}
    KL\left( q_{\phi'}(\mathbf{\hat{S}} \mid {\mathbf{\overline{X}}}) \parallel p(\mathbf{\hat{S}}) \right) 
    &=-H(q_{\phi'}(\mathbf{\hat{S}})) + \sum_{ij}\log{\frac{1}{\beta_{ij}}} \\
    &\leq \sum_{ij}\log{\frac{1}{\beta_{ij}}},\\
\end{aligned}
\end{equation}
where $H(\cdot)$ denotes entropy and $\boldsymbol{\beta}$ is determined by the observed graphs and corresponding beliefs (see Eq.~\eqref{eqBeta}). From this bound, we can know that the maximum variational information in the variational consensus graph is dependent on the consistency of these observed graphs and how relative they are to our clustering task. For example, when all graphs are consistent with $\{a^v_{ij}=1\}^V_{v=1}$, we could get $\log{\beta_{ij}=1}$, which enforces $H(q(\hat{s}_{ij}))=0$.

\subsection{Global and Specific Graph Encoding}
\label{secGraphEnc}
In this section, we will propose our graph encoder, which contains three inputs, \emph{i.e.}, observed graph data $\mathbf{A}^v$ with its attributed features $\mathbf{X}^v$, and a variational consensus graph $\mathbf{\hat{S}}$. The aim of this graph encoder is to inject the consensus graph structure information into each specific view and generate embeddings. So this module could be decomposed into two steps: 1) message-passing-based graph embedding (denoted by $\psi(\cdot)$) to embed the specific and consensus graph structure, respectively; 2) low-dimensional representation learning (denoted by $f_v(\cdot)$) to fuse view-specific graph embeddings and consensus graph embeddings while obtaining a lower-dimensional representation.

\subsubsection{Message-Passing-Based Graph Embedding}
Inspired by previous GNNs~\cite{SGC,S2GC,GCN2}, we remove most redundant parameters in GNNs and introduce a simple yet effective message-passing method in this part.

Let $\mathbf{E}^{(l)}$ be the embedding of $l$-th layer, the definition of vanilla GNN~\cite{GCN} is as follows:
\begin{equation}
\nonumber
    \mathbf{E}^{(l+1)} = (\mathbf{\widetilde{A}}\mathbf{E}^{(l)}\mathbf{W}^{(l)}), \text{where } \mathbf{E}^{(0)} = \mathbf{X}.
\end{equation}
We first remove the unnecessary parameters $\mathbf{W}$ from this equation, and then, add a residual connection to every layer from the input features $\mathbf{X}$. Therefore, the improved parameter-free graph neural network could be written as:
\begin{equation}
\nonumber
    \mathbf{E}^{(l+1)} = \mathbf{\widetilde{A}}\mathbf{E}^{(l)} + \mathbf{X}.
\end{equation}
Finally, we can obtain our message-passing method via mathematical induction:
\begin{equation}
\label{eqMessagePassing}
    \psi(\mathbf{X}, \mathbf{\widetilde{A}})=\left(\sum^{order}_{l=0}\mathbf{\widetilde{A}}^l+\mathbf{I}\right)\mathbf{X},
\end{equation}
where $order$ is a hyper-parameter which denotes the aggregation orders on graph $\mathbf{A}$, and $\mathbf{I}$ denotes the identity matrix. $\mathbf{\widetilde{A}}$ and $\mathbf{X}$ are normalized graphs and their attributed features, respectively.

\subsubsection{Low-Dimensional Representation} A straightforward instantiation of $f_v(\mathbf{\widetilde{A}}, \mathbf{\hat{S}}, \mathbf{X})$ is MLP. Concretely, we concatenate the specific graph embeddings and consensus graph embeddings and feed them into this MLP to get a lower-dimensional embedding $\mathbf{Z}^v \in \mathbb{R}^{n \times D}$:
\begin{equation}
\begin{aligned}
    \mathbf{Z}^v
    &=f_v(\mathbf{\widetilde{A}}^v, \mathbf{\hat{S}},\mathbf{X}^v)\\
    &= f_v\left(Concat(\psi(\mathbf{X}^v, \mathbf{\widetilde{A}}^v), \psi(\mathbf{X}^v, \mathbf{\hat{S}}))\right),
\end{aligned}
\end{equation}
Here, the propagation of $\mathbf{X}^v$ over $\psi$ and $f_v$ will lead to some information loss. To reserve more view-specific feature information, we maximize the mutual information $I(\mathbf{X}^v, \mathbf{Z}^v)$:
\begin{equation}
     \underset{\phi^v}{\arg\max} I(\mathbf{Z}^v,\mathbf{X}^v),
\end{equation}
where $\phi^v$ represents the learnable parameters in $f_v$.
\subsection{Multi-View Fusion for Clustering Task}
\label{SecMVFusion}
\subsubsection{Multi-View Fusion}
To effectively fuse all views' embeddings $\{\mathbf{Z}^v\}^V_{v=1}$ for the clustering task, a natural approach is to assign a task-relevant weight to each view. According to this intuition, we assign beliefs $\{b^v\}^V_{v=1}$ for $V$ views. These beliefs imply how the corresponding view is relative to the clustering task, and their values are optimized in the training process by a self-supervised strategy (see Section~\ref{sec:beliefsOpti}). After the beliefs are obtained, we reweight every view and concatenate them to get final global representations $\mathbf{\overline{Z}}\in \mathbb{R}^{n \times (V \cdot D)}$:
\begin{equation}
\label{eqFinalZ}
    \mathbf{\overline{Z}}=Concat(b^{\it1}\mathbf{Z}^{\it1}, b^{\it2}\mathbf{Z}^{\it2}, ..., b^{\it V}\mathbf{Z}^V).
\end{equation}
In the decision process, the optimized beliefs are directly used to obtain the global representation $\mathbf{\overline{Z}}$ via Eq.~\eqref{eqFinalZ} which is then fed into $k$-means for clustering. Next, we will introduce how the beliefs are optimized in the training process.
\subsubsection{Beliefs Updating}
\label{sec:beliefsOpti}
To effectively find the beliefs for each view, we propose a self-supervised strategy. Firstly, all beliefs are initialized to $1$, \emph{i.e.}, $\{b^{(0),v} \mid b^{(0),v}=1\}^V_{v=1}$, where $b^{(0),v}$ denotes the initial belief of $v$-th view. Then, in $t$-th training epoch, $\mathbf{\overline{Z}}^{(t-1)}$ is generated via Eq.~\eqref{eqFinalZ} and fed into $k$-means on to obtain the pseudo-labels of this epoch. These pseudo-labels are used to evaluate the quality of each view's embeddings $\mathbf{\overline{Z}}^{(t-1)}$. Specifically, each view's embeddings $\mathbf{\overline{Z}}^{(t-1)}$ are simultaneously fed into $k$-means to get view-specific predictions. The pseudo-labels from $\mathbf{\overline{Z}}^{(t-1)}$ are considered as ground-truth and other predicted labels from $\{\mathbf{Z}^v\}^V_{v=1}$ are considered as prediction results to compute each view's clustering score (\emph{e.g.}, using normalized mutual information), \emph{i.e.}, $\{score^{(t),v}\}^V_{v=1}$. These scores are then normalized for updating beliefs:
\begin{equation}
\label{eqBelief}
    b^{(t),v} = \left(\frac{score^{(t),v}}{\text{max}(score^{(t),1}, score^{(t),2}, ..., score^{(t),V})}\right)^\rho,
\end{equation}
where $\rho \geq 0$ is a soft parameter. When $\rho=0$, we can get $b^{(t),v}=1$ for all $t$ and $v$, which denotes that all graphs are equally believed; when $\rho \rightarrow \infty$, $\{b^{(t),v}\}^V_{v=1}$ are binarized.
Finally, we can generate $\mathbf{\overline{Z}}^{(t)}$ via Eq.~\eqref{eqFinalZ} under the new beliefs $\{b^{(t),v}\}_{v=1}^V$, which will be utilized to train the model and generate a more correct belief in next epoch.

\subsubsection{Multi-View Clustering Loss}
The clustering loss is widely used in the traditional clustering task~\cite{MAGCN,DEC,xu2021deep}, which encourages the assignment distribution of samples in the same cluster to be more similar. Concretely, let $Q^v$ be a soft assignment calculated by Student's $t$-distribution~\cite{tsne} of $v$-th view, and $P^v$ be a target distribution calculated by sharpening the soft assignment $Q^v$. For MGC, we encourage each view's soft distribution to fit the global representation's target distribution by KL divergence:
\begin{equation}
    \label{eq:KLloss}
    \mathcal{L}_{c} = \sum^V_{v=1} KL(\overline{P} \Vert Q^v) + KL(\overline{P} \Vert \overline{Q}),
\end{equation}
where $\overline{P}$ and $\overline{Q}$ denote the target and soft distribution of global representation $\mathbf{\overline{Z}}$.

\subsection{Objective Function}
\label{secObj}
Our objective function contains three parts: maximizing ELBO ($\mathcal{L}_{E}$) to optimize the variational graph generator; maximizing mutual information ($I(\mathbf{Z}^v, \mathbf{X}^v)$ and $I(\mathbf{Q}, \mathbf{\overline{X}})$) to preserve feature information; and minimizing clustering loss ($\mathcal{L}_{c}$) to optimize the task-related clustering results. $\gamma_c$ and $\gamma_E$ are introduced to trade off the numerical value of $\mathcal{L}_c$ and $\mathcal{L}_E$.
Formally, we have the objective function:
\begin{equation}
\label{eqFinalLoss}
\begin{aligned}
    \underset{\phi',\eta', \{\phi^v\}^V_{v=1},\{\eta^v\}^V_{v=1}}{\arg\min}  \mathcal{L}_r + \gamma_c\mathcal{L}_{c} - \gamma_E\mathcal{L}_{E}.
\end{aligned}
\end{equation}
Here, $\mathcal{L}_r$ is an instantiation of negative lower bound of $I(\mathbf{Z}^v, \mathbf{X}^v)$ and $I(\boldsymbol{\alpha}, \mathbf{\overline{X}})$ (see Section \ref{secLBMI}), and  $\{\eta^v\}^V_{v=1}$ and $\eta'$ denote corresponding learnable parameters. Notably, the parameters of $p_{\xi^v}(\mathbf{A}^v \mid \mathbf{\hat{S}},\mathbf{X}^v)$ and $f_v$ are shared in our model, so $\{\xi^v\}^V_{v=1}$ is omitted in Eq.~\eqref{eqFinalLoss}.
More implementation details of each term can be found in Appendix~\ref{appLossFunction}.

\subsection{Theoretical Analysis}
\label{secTheoAna}
In this subsection, we propose two theorems for demonstrating that the consensus and task relevance are implied in the generated consensus graph at first (see Section~\ref{secConsensus}). Then, we introduce the variational graph information bottleneck ($\mathcal{IB}_{graph}$), and prove the connection between our objective function with $\mathcal{IB}_{graph}$. This can further explain what kind of information is learned by the consensus graph (see Section~\ref{secIBvar}). Finally, we compute the lower bound of mutual information to optimize the proposed $I(\mathbf{Z}^v, \mathbf{X}^v)$ and $I(\boldsymbol{\alpha}, \mathbf{\overline{X}})$ (see Section~\ref{secLBMI}).

\subsubsection{Uncertainty of Consensus Graph}
\label{secConsensus}
Let $H(\mathbf{\hat{S}})$ denote the uncertainty of $\mathbf{\hat{S}}$, and $d_{Ham}(\mathbf{A}^{\it1}, \mathbf{A}^{\it2})$ denote the Hamming distance, which quantizes the inconsistency of graph $\mathbf{A}^{\it1}$ and $\mathbf{A}^{\it2}$. We have the following theorems:
\begin{theorem}[Consensus among graphs]
\label{theorem1}
Assume all graphs are equally reliable, the more consistent the observed graphs are, the less uncertain the consensus graph is, and vice versa:
\begin{equation}
    H(\mathbf{\hat{S}}) \propto d_{Ham}(\mathbf{A}^{\it1}, \mathbf{A}^{\it2}).
\end{equation}
\end{theorem}

\begin{proof}[\textbf{Proof}]
Assume there are two views with equal belief, i.e., $V=2$ and $b^{\it1}=b^{\it2}=b$, we have:
\begin{equation}
\nonumber
\begin{aligned}
 \log {\beta_{ij} }  &= \log { \frac{ b (a^{\it1}_{ij} + a^{\it2}_{ij}) + (1-b)(2 - (a^{\it1}_{ij} + a^{\it2}_{ij})) }{2b} }.\\
\end{aligned}
\end{equation}
Note that KL divergence is always positive, so we have :
\begin{equation}
\small
\nonumber
\begin{aligned}
	 &KL\left( q_{\phi'}(\mathbf{\hat{S}} \mid {\mathbf{\overline{X}}}) \parallel p(\mathbf{\hat{S}}) \right) 
	 = -H(\mathbf{\hat{S}}) + \sum_{ij}\log{\frac{1}{\beta_{ij}}} \geq 0 \\
    &\Rightarrow  H(\mathbf{\hat{S}}) \leq \sum_{ij}\log{\frac{1}{\beta_{ij}}} \\
    &=\sum_{ij}\log\frac{2b}{{b(a^{\it1}_{ij} + a^{\it2}_{ij}) + (1-b)(2- (a^{\it1}_{ij} + a^{\it2}_{ij})) }}\\
    &=\sum_{ij}\log{2b} - \sum_{ij}C_{ij},
\end{aligned}
\end{equation}
where $C_{ij}=\log{{b(a^{\it1}_{ij} + a^{\it2}_{ij}) + (1-b)(2- (a^{\it1}_{ij} + a^{\it2}_{ij})) }}$, and we have:
\begin{equation}
\nonumber
    C_{ij}=
    \begin{cases}
     \log{(b+1-b)}=0, &\text{if }a^{\it1}_{ij}\neq a^{\it2}_{ij},\\
      \log{2b},     &\text{if } a^{\it1}_{ij} = a^{\it2}_{ij}=1,\\
      \log{2(1-b)}, &\text{if } a^{\it1}_{ij} = a^{\it2}_{ij}=0.
    \end{cases}
\end{equation}
Let $n$, $t$ and $k$ be the numbers of all edges, $a^{\it1}_{ij} = a^{\it2}_{ij}=1$ and $a^{\it1}_{ij} = a^{\it2}_{ij}=0$ respectively, so we have:
\begin{equation}
\nonumber
     H(\mathbf{\hat{S}}) \leq n\log{2b} - t\log{2b} - k\log{2(1-b)}.
\end{equation}
Considering $d_{Ham}=\sum_{ij} \vert a^{\it1}_{ij} - a^{\it2}_{ij} \vert = n-t-k$, we have:
\begin{equation}
\nonumber
     H(\mathbf{\hat{S}}) \leq d_{Ham}\log{2b} + k\log{\frac{b}{1-b}}. 
\end{equation}
So, the maximized $H(\mathbf{\hat{S}})$ is proportional to $d_{Ham}$, i.e., $H(\mathbf{\hat{S}})\propto d_{Ham}$.
\end{proof}
Using $H(\mathbf{A}^v, \mathbf{\hat{S}})$ (\emph{e.g.}, cross entropy) to depict the similarity between the distribution of $\mathbf{\hat{S}}$ and $\mathbf{A}^v$, and the belief $b^v \in (0,1]$ can represent the relevance of $v$-th view to our clustering task, we have:
\begin{theorem}[Task relevance of each graph]
\label{theorem2}
For a specific graph, the more relevant the graph is to the clustering task, the greater the similarity between the distribution of $\mathbf{\hat{S}}$ and this graph, and vice versa:
\begin{equation}
    H(\mathbf{A}^v, \mathbf{\hat{S}}) \propto b^v.
\end{equation}
\end{theorem}
\begin{proof}[\textbf{Proof}]
\begin{equation}
\nonumber
\begin{aligned}
    &H(\mathbf{A}^v, \mathbf{\hat{S}})\\
    &= \sum_{ij} a^v_{ij}\log\frac{1}{p(\hat{s}_{ij})} + (1-a^v_{ij})\log\frac{1}{1-p(s_{ij})}\\
    &=\sum_{i,j \in U}\log{\frac{1}{\beta_{ij}}} + \sum_{i,j \in T}\log\frac{1}{1-\beta_{ij}},\\
\end{aligned}
\end{equation}
where $U=\{i,j \mid a^v_{ij} = 1\}$ and $T=\{i,j \mid a^v_{ij} = 0\}$. Therefore, we have:
\begin{equation}
\nonumber
\begin{aligned}
    &H(\mathbf{A}^v, \mathbf{\hat{S}})\\
    &=\sum_{i,j \in U}\log{\frac{\sum^V_{k=1}b^k}{b^va^v_{ij}+(1-b^v)(1-a^v_{ij})}} \\
    &\quad \quad + \sum_{i,j \in T}\log\frac{\sum^V_{k=1}b^k}{\sum^V_{k=1}b^k - b^va^v_{ij}-(1-b^v)(1-a^v_{ij})}\\
    &= \sum_{ij}\log{(\sum^V_{k=1}b^k)} - (\sum_{i,j \in U}\log{b^v} + \sum_{i,j \in T}\log{(b^v -1 + \sum^V_{k=1}b^k}))\\
    &\Rightarrow H(\mathbf{A}^v, \mathbf{\hat{S}}) \propto b^v.
\end{aligned}
\end{equation}
From this equation, we can know that the more the $v$-th graph relative to the clustering task \emph{i.e.}, $b^v$ is larger, the more the $\hat{\mathbf{S}}$ is similar to this graph, \emph{i.e.}, $H(\mathbf{A}^v, \hat{\mathbf{S}})$ is smaller. 
\end{proof}
From Theorem~\ref{theorem1} and Theorem~\ref{theorem2}, we can conclude the following corollary:
\begin{corollary}
\label{corollary1}
The distribution of $\mathbf{\hat{S}}$ is dependent on the consistency of observed graphs and how much they are related to the clustering task.
\end{corollary}



\subsubsection{Variational Graph Information Bottleneck}
\label{secIBvar}
To understand how the variational graph generator infers a graph with sufficient information, we build a connection between the objective function of the variational graph generator (Section~\ref{secVGG}) and Information Bottleneck (IB) principle~\cite{DVIB,LRusingIB}. Recall the definition of supervised IB~\cite{IB}:
\begin{definition}[Supervised IB]
 The supervised IB is to maximize the Information Bottleneck Lagrangian:
\begin{equation}
    \mathcal{IB}_{sup}= I(Y,Z_X) - \omega I(X, Z_X), \text{where }\omega > 0.
\end{equation}
\end{definition}
This shows that the supervised IB aims to maximize the mutual information between latent representation $Z_X$ and corresponding target labels $Y$, meanwhile trying to compress more information from $X$ (\emph{i.e.}, minimize the mutual information between $X$ and $Z_X$). The intuition behind IB is that the latent variable $Z_X$ tries to collect less but sufficient information from $X$ to facilitate the task.

Recall the proposed method, VGMGC aims at inferring the graph $\mathbf{\hat{S}}$ from global features $\mathbf{\overline{X}}$ in the inference process, and generating the observed graphs $\mathbf{A}^v$ in the generative process. Intuitively, only the topological information in $\mathbf{\overline{X}}$ is useful and most others are redundant. Therefore, we can obtain our variational graph IB:
\begin{definition}[Variational graph IB]
\label{def2}
The IB in the variational graph generator is:
\begin{equation}
    \mathcal{IB}_{graph} = \sum\nolimits_{v} I(\mathbf{A}^v, \mathbf{\hat{S}}) - \omega I(\mathbf{\hat{S}}, \mathbf{\overline{X}}), \text{where }\omega > 0.
\end{equation}
\end{definition} 
\begin{theorem}
\label{theorem3}
Maximizing the ELBO (Eq.~\eqref{eqELBO}) is equivalent to maximizing variational graph IB:
\begin{equation}
\label{eqLeIB}
    \arg\max \mathcal{L}_{E} \cong \arg\max \mathcal{IB}_{graph}.
\end{equation}
\end{theorem}
\begin{proof}[\textbf{{Proof}}]
Here, $\mathcal{IB}_{graph}$ is defined in Definition~\ref{def2}.
For simplification, we rewrite mutual information as $I(\cdot; \cdot)$; $\mathbf{A}$, $\mathbf{\overline{X}}$, $\mathbf{X}^v$ and $\mathbf{\hat{S}}$ as $A$, $\overline{X}$, $X^v$ and $\hat{S}$.
\\
Recall the property of mutual information:
\begin{equation}
\nonumber
\begin{aligned}
I(X;Y \mid Z) &= I(X;Y) + H(Z \mid X) + H(Z \mid Y)\\
&- H(Z\mid X,Y) - H(Z),
\end{aligned}
\end{equation}
where, $H(\cdot)$ denotes Entropy. From this property, we have:
\begin{equation}
\nonumber
\small
\begin{aligned}
    &I(A^v; \hat{S}\mid \overline{X}, X^v)
    = I(A^v; \hat{S}) + H(\overline{X},X^v \mid \hat{S})\\
    & \quad + H(\overline{X},X^v\mid A^v)  - H(\overline{X},X^v \mid A^v,\hat{S}) - H(\overline{X}, X^v).
\end{aligned}
\end{equation}
Note that ${\overline{X}}$ and ${X}^v$ are observed, so $H({\overline{X}})$ and $H({X}^v)$ are two constants. We have:
$$
I(A^v; \hat{S}) \propto I(A^v; \hat{S}\mid \overline{X}, X^v).
$$

Let $p_D(a^v)$ denote data distribution of $a^v$, from the Bayesian network (Figure~\ref{fig:BayesianNet}), we have the following derivation:
\begin{equation}
\small
\begin{aligned}
\label{eq:theorem2.1}
  &I(A^v; \hat{S}\mid \overline{X}, X^v)\\
  &= \mathbb{E}_{p(\overline{x},x^v)}\int\int p(a^v, \hat{s}\mid \overline{x}, x^v)\log{\frac{p(a^v, \hat{s} \mid \overline{x}, x^v)}{p_D(a^v \mid \overline{x}, x^v) q(\hat{s} \mid \overline{x}, x^v)}} d\hat{s} da^v\\
  &=\mathbb{E}_{p(\overline{x})p(x^v)}\int\int p_D(a^v)q(\hat{s}\mid \overline{x})\log{\frac{p(a^v \mid \hat{s}, x^v)}{p_D(a^v)}} d\hat{s} da^v\\
  &=\mathbb{E}_{p(\overline{x})p(x^v)}\int\int p_D(a^v)q(\hat{s}\mid \overline{x})\log{p(a^v \mid \hat{s}, x^v)} d\hat{s}da^v\\
  &\quad \quad+ H(A^v)\\
  &\geq \mathbb{E}_{p(\overline{x})p(x^v)}\int\int p_D(a^v)q(\hat{s}\mid \overline{x})\log{p(a^v \mid \hat{s}, x^v)} d\hat{s}da^v\\
  &\geq \mathbb{E}_{p(\overline{x})p(x^v)} \mathbb{E}_ {q(\hat{s}\mid \overline{x})}\log{p_{\xi^v}(a^v \mid \hat{s}, x^v)},
\end{aligned}
\end{equation}
where $p_{\xi^v}(a^v \mid \hat{s}, x^v)$ is a variational approximation of $p(a^v \mid \hat{s}, x^v)$. The last term of Eq.~\eqref{eq:theorem2.1} is deduced from $KL(p(a^v \mid \hat{s}, x^v)\Vert p_{\xi^v}(a^v \mid \hat{s}, x^v))\geq 0$.
Therefore, we have obtained the first term in $\mathcal{L}_E$ from Eq.~\eqref{eq:theorem2.1}. For the second term, we have:
\begin{equation}
\begin{aligned}
\label{eq:theorem2.2}
    I(\hat{S}, \overline{X}) 
    &= \int \int q(\hat{s},\overline{x}) \log{\frac{q(\hat{s},\overline{x})}{q(\hat{s})p(\overline{x})}} d\overline{x} d\hat{s}\\
    &=\mathbb{E}_{p(\overline{x})}\int q(\hat{s} \mid \overline{x})\log{\frac{q(\hat{s}\mid \overline{x})}{q(\hat{s})} \frac{p(\hat{s})}{p(\hat{s})}} d\hat{s}\\
    &\leq \mathbb{E}_{p(\overline{x})}KL\left(q(\hat{s}\mid \overline{x})\Vert p(\hat{s})\right).
\end{aligned}
\end{equation}
Here, $p(\hat{s})$ and $p(\overline{x})$ denote the prior distribution of $\hat{s}$ and data distribution of $\overline{x}$, and $q(\cdot)$ denotes the distribution can be approximated. This upper bound is the second term in $\mathcal{L}_E$.
So, Theorem~\ref{theorem3} can be obtained by concluding the above formulas:
\begin{equation}
\nonumber
\small
\begin{aligned}
    \arg\max I(A^v,\hat{S}) - \omega I(\hat{S}, \overline{X}) &\cong \arg\max  \mathbb{E}_ {q(\hat{s}\mid \overline{x})}\log{p_{\xi^v}(a^v \mid \hat{s}, x^v)}\\
    &\quad- \omega KL\left(q(\hat{s}\mid \overline{x})\Vert p(\hat{s})\right)\\
    \Rightarrow  \arg\max \mathcal{IB}_{graph} &\cong  \arg\max \mathcal{L}_{E}.
\end{aligned}
\end{equation}
\end{proof}
Eq.~\eqref{eqLeIB} indicates that by optimizing Eq.~\eqref{eqELBO}, the consensus graph $\mathbf{\hat{S}}$ can capture the most useful topological information and discard redundant information from global feature $\mathbf{\overline{X}}$.

\subsubsection{Lower Bound of Mutual Information}
\label{secLBMI}
Following \cite{MINE,VBMI}, we analyze the lower bound of mutual information to optimize $I(\mathbf{Z}^v,\mathbf{X}^v)$ and $I(\boldsymbol{\alpha}^v,\mathbf{\overline{X}})$.
\begin{definition}
 The mutual information $I(\mathbf{Z}^v,\mathbf{X}^v)$ is:
\begin{equation}
\begin{aligned}
\label{eqDefIM}
    I(\mathbf{Z}^v, \mathbf{X}^v) 
    &= \mathbb{E}_{p(\mathbf{z}^v, \mathbf{x}^v)} \log{\frac{p(\mathbf{z}^v,\mathbf{x}^v)}{p(\mathbf{z}^v)p(\mathbf{x}^v)}}.
\end{aligned}
\end{equation}
\end{definition}
Considering that the entropy $H(\cdot)$ is always positive, we have:
\begin{equation}
\nonumber
\begin{aligned}
    &I(\mathbf{Z}^v, \mathbf{X}^v) \\
    &= \int \int p(\mathbf{z}^v, \mathbf{x}^v)\log{\frac{p(\mathbf{x}^v \mid \mathbf{z}^v)}{p(\mathbf{x}^v)}}d\mathbf{x}^v d\mathbf{z}^v\\
    &= \int\int p(\mathbf{x}^v)p(\mathbf{z}^v \mid \mathbf{x}^v)\log{p(\mathbf{x}^v \mid \mathbf{z}^v)} d\mathbf{x}^v d\mathbf{z}^v \\
    &\quad \quad + \int\int p(\mathbf{x}^v)p(\mathbf{z}^v \mid \mathbf{x}^v) \log{\frac{1}{p(\mathbf{x}^v)}} d\mathbf{x}^v d\mathbf{z}^v\\
    &= \int\int p(\mathbf{x}^v)p(\mathbf{z}^v \mid \mathbf{x}^v)\log{p(\mathbf{x}^v \mid \mathbf{z}^v)} d\mathbf{x}^v d\mathbf{z}^v + H(\mathbf{x}^v)\\
    &\geq \int \int p(\mathbf{x}^v)p(\mathbf{z}^v \mid \mathbf{x}^v)\log{p(\mathbf{x}^v \mid \mathbf{z}^v)} d\mathbf{x}^v d\mathbf{z}^v.
\end{aligned}
\end{equation}
Here, because $p(\mathbf{x}^v \mid \mathbf{z}^v)$ is intractable, $q_{\eta^v}(\mathbf{x}^v \mid \mathbf{z}^v)$ is introduced as a variational approximation of $p(\mathbf{x}^v \mid \mathbf{z}^v)$. Considering that KL divergence is always positive, i.e., $KL(p(\mathbf{x}^v \mid \mathbf{z}^v) \Vert q_{\eta^v}(\mathbf{x}^v \mid \mathbf{z}^v)) \geq 0$, we can obtain the lower bound of mutual information:
\begin{equation}
\nonumber
\begin{aligned}
&\int \int p(\mathbf{x}^v)p(\mathbf{z}^v \mid \mathbf{x}^v)\log{p(\mathbf{x}^v \mid \mathbf{z}^v)} d\mathbf{x}^v d\mathbf{z}^v \\
&= \int \int p(\mathbf{x}^v)p(\mathbf{z}^v \mid \mathbf{x}^v)\log{\frac{p(\mathbf{x}^v \mid \mathbf{z}^v)}{q_{\eta^v}(\mathbf{x}^v \mid \mathbf{z}^v)}q_{\eta^v}(\mathbf{x}^v \mid \mathbf{z}^v)} d\mathbf{x}^v d\mathbf{z}^v\\
&= \int \int p(\mathbf{z}^v)KL(p(\mathbf{x}^v \mid \mathbf{z}^v) \Vert q_{\eta^v}(\mathbf{x}^v \mid \mathbf{z}^v)) d\mathbf{x}^v d\mathbf{z}^v\\
&\quad \quad + \int\int p(\mathbf{x}^v)p(\mathbf{z}^v \mid \mathbf{x}^v)\log{q_{\eta^v}(\mathbf{x}^v \mid \mathbf{z}^v)}d\mathbf{x}^v d\mathbf{z}^v\\
& \geq \int\int p(\mathbf{x}^v)p(\mathbf{z}^v \mid \mathbf{x}^v)\log{q_{\eta^v}(\mathbf{x}^v \mid \mathbf{z}^v)}d\mathbf{x}^v d\mathbf{z}^v\\
&= \mathbb{E}_{ p(\mathbf{x}^v)p(\mathbf{z}^v \mid \mathbf{x}^v)}\log{q_{\eta^v}(\mathbf{x}^v \mid \mathbf{z}^v)}.
\end{aligned}
\end{equation}

Therefore, the mutual information here can be maximized by maximizing its lower bound:
\begin{equation}
    \arg\max I(\mathbf{Z}^v, \mathbf{X}^v) \cong \arg\max \mathbb{E}_{ p(\mathbf{x}^v)p(\mathbf{z}^v \mid \mathbf{x}^v)}\log{q_{\eta^v}(\mathbf{x}^v \mid \mathbf{z}^v)}.
\end{equation}
In our case, $x_j^v$ is subordinated to Bernoulli distribution, so $-\mathbb{E}_{p(\mathbf{z}^v \mid \mathbf{x}^v)p(\mathbf{x}^v)}\log{q_{\eta^v}(\mathbf{x}^v \mid \mathbf{z}^v)}$ can be instantiated by a binary cross entropy loss between $\mathbf{x}^v$ and its reconstruction $\mathbf{\breve{x}}^v$. The lower bound of all $I(\mathbf{Z}^v,\mathbf{X}^v)$ and $I(\boldsymbol{\alpha},\mathbf{\overline{X}})$ are added together, written as $\mathcal{L}_{r}$.
\begin{table}[]
    \centering
    \caption{The statistics information of experimental datasets.}
    \begin{tabular}{lllll}
    \toprule[1.5pt]
        Dataset & \#Clusters & \#Nodes & \#Features & \makecell[l]{Graphs} \\
        \midrule
        \multirow{2}*{ACM} & \multirow{2}*{3} & \multirow{2}*{3025} & \multirow{2}*{1830} & $\mathcal{G}^1$ co-paper \\
         & & & & $\mathcal{G}^2$ co-subject \\
         \midrule
         \multirow{3}*{DBLP} & \multirow{3}*{4} & \multirow{3}*{4057} & \multirow{3}*{334} & $\mathcal{G}^1$ co-author \\
         & & & & $\mathcal{G}^2$ co-conference \\
         & & & & $\mathcal{G}^3$ co-term \\
         \midrule
         \multirow{2}*{Photos} & \multirow{2}*{8} & \multirow{2}*{7487} & 745 & \multirow{2}*{$\mathcal{G}^1$ co-purchase} \\
          & & & 7487 &  \\
         \midrule
         \multirow{2}*{Computers} & \multirow{2}*{10} & \multirow{2}*{13381} & 767 & \multirow{2}*{$\mathcal{G}^1$ co-purchase} \\
          & & & 13381 &  \\
          \midrule
          Cora & 7 & 2708   & 1433  & $\mathcal{G}^1$ citation network\\
          \midrule
          Citeseer & 6 & 3327   & 3703  & $\mathcal{G}^1$ citation network\\
          \midrule
          \multirow{2}*{BBC Sport} & \multirow{2}*{5} & \multirow{2}*{544} & 3283 &$\mathcal{G}^1$ kNN graph\\
          & & & 3183  & $\mathcal{G}^2$ kNN graph \\
          \midrule
          \multirow{3}*{3Sources} &\multirow{3}*{6} & \multirow{3}*{169} & 3560 & $\mathcal{G}^1$ kNN graph\\
          &  &  & 3631 & $\mathcal{G}^2$ kNN graph\\
          &  &  & 3068  & $\mathcal{G}^3$ kNN graph\\
         \bottomrule[1.5pt]
    \end{tabular}
    \label{tab:datasets}
\end{table}
\begin{table}[t]
\setlength\tabcolsep{2.4pt}
    \centering
    \caption{The overall clustering results.
        The values in the table are shown as percentage. The best and second-best results are shown in bold and underlined, and `$-$' denotes the results not reported in the original papers.
        }
    \begin{tabular}{r|cccc|cccc}
    \toprule[1.5pt]
    \multirow{2}*{Methods / Datasets} & \multicolumn{4}{c|}{Amazon photos} & \multicolumn{4}{c}{Amazon computers} \\
         & NMI & ARI & ACC & F1 & NMI & ARI & ACC & F1 \\
    \midrule
     MAGCN \cite{MAGCN} (2020) & 39.0 & 24.0 & 51.7 & 47.3 & $-$ & $-$ & $-$ & $-$ \\
    COMPLETER \cite{lin2021completer} (2021) & 26.1 & 7.6 & 36.8 & 30.7 & 15.6 & 5.4 & 24.2 & 16.0 \\
    MVGRL \cite{hassani2020mvgrl} (2021) & 43.3 & 23.8 & 50.5 & 46.0 & 10.1 & 5.5 & 24.5 & 17.1 \\
    MvAGC \cite{lin2021graph} (2021) & 52.4 & 39.7 & 67.8 & 64.0 & 39.6 & 32.2 & 58.0 & 41.2 \\
    MCGC \cite{pan2021multi} (2021) & \underline{61.5} & \underline{43.2} & \underline{71.6} & \underline{68.6} & \underline{53.2} & \underline{39.0} & \underline{59.7} & \textbf{52.0} \\
    \midrule
    \textbf{VGMGC (ours)} & \textbf{66.8} & \textbf{58.4} & \textbf{78.5} & \textbf{76.9} & \textbf{53.5} & \textbf{47.5} & \textbf{62.2} & \underline{50.2} \\
    \midrule[1pt]
    \multirow{1}*{Methods / Datasets} & \multicolumn{4}{c|}{ACM} & \multicolumn{4}{c}{DBLP} \\
    \midrule
    RMSC \cite{xia14RMSC} (2014) & 39.7 & 33.1 & 63.2 & 57.5 & 71.1 & 76.5 & 89.9 & 82.5 \\
    LINE \cite{LINE} (2015) & 39.4 & 34.3 & 64.8 & 65.9 & 66.8 & 69.9 & 86.9 & 85.5 \\
    GAE \cite{GAE} (2016) & 49.1 & 54.4 & 82.2 & 82.3 & 69.3 & 74.1 & 88.6 & 87.4 \\
    PMNE \cite{PMNE} (2017) & 46.5 & 43.0 & 69.4 & 69.6 & 59.1 & 52.7 & 79.3 & 79.7 \\
    SwMC \cite{nie2017SwMC} (2017) & 8.4 & 4.0 & 41.6 & 47.1 & 37.6 & 38.0 & 65.4 & 56.0 \\
    MNE \cite{zhan2018MNE} (2018) & 30.0 & 24.9 & 63.7 & 64.8 & $-$ & $-$ & $-$ & $-$ \\
    O2MAC \cite{o2multi} (2020) & 69.2 & 73.9 & 90.4 & 90.5 & 72.9 & 77.8 & 90.7 & 90.1 \\
    MvAGC \cite{lin2021graph} (2021) & 67.4 & 72.1 & 89.8 & 89.9 & 77.2 & 82.8 & 92.8 & 92.3 \\
    MCGC \cite{pan2021multi} (2021) & 71.3 & 76.3 & 91.5 & 91.6 & \textbf{83.0} & 77.5 & 93.0 & 92.5 \\
    DuaLGR \cite{dulgr} (2023)& 73.2 & 79.4 & 92.7 & 92.7 & 75.5 & 81.7 &  92.4 & 91.8 \\
    R\textsuperscript{2}FGC \cite{R2FGC} (2023)& 72.4 & 78.7 & 92.4 & 92.5 & 50.8 & 56.3 & 81.0 & 80.5\\
    BTGF \cite{BTGF} (2024)& \underline{75.8} & \underline{80.9} & \underline{93.2} & \underline{93.3} & 62.4 & 59.7 & 83.1 & 83.8 \\
    DIAGC \cite{DIAGC} (2024)& 71.6 & 77.0 & 91.8 & 91.8 & \underline{78.3} & \textbf{83.7} &  \textbf{93.3} & \textbf{92.8} \\
    \midrule
    \textbf{VGMGC (ours)}& \textbf{76.3} & \textbf{81.9} & \textbf{93.6} & \textbf{93.6} & \underline{78.3} & \textbf{83.7} & \underline{93.2} & \underline{92.7} \\
    \bottomrule[1.5pt]
    \end{tabular}
    \label{tab:overall_results}
\end{table}
\section{Experiments}
\label{secExp}
\subsection{Experimental Settings}
\subsubsection{Datasets}
Four common benchmark multi-view graph datasets, \emph{i.e.}, ACM, DBLP, Amazon photos, and Amazon computers, two traditional single-view graph datasets Cora and Citeseer, and two traditional multi-view datasets BBC Sport and 3Sources, are selected to evaluate the proposed VGMGC. The statistics information of these datasets is listed in Table~\ref{tab:datasets}. Specifically,
\begin{itemize}
    \item ACM\footnote{\url{https://dl.acm.org/}}: This is a paper network from the ACM database. It consists of two different graphs, co-paper and co-subject, representing the relationship of the same author with papers and the relationship of the same subject with papers respectively. The node feature is a bag-of-words representation of each paper's keywords.
    \item DBLP\footnote{\url{https://dblp.uni-trier.de/}}: This is an author network from the DBLP database. The dataset includes three graphs, \emph{i.e.} co-author (co-authors for the same paper), co-conference (papers published on the same conference) and co-term (papers published on the same term). The node feature is a bag-of-words representation of each author's keywords.
    \item Amazon datasets: Amazon photos and Amazon computers are the subsets of the Amazon co-purchase network from~\cite{pan2021multi}. The network means the two goods that are purchased together. The node feature is a bag-of-words representation of each good's product reviews. Due to Amazon photos and Amazon computers only having single-view features, we constructed the second view features through Cartesian product following by \cite{MAGCN}.  
    \item Cora\footnote{\url{https://graphsandnetworks.com/the-cora-dataset/}} and Citeseer\footnote{\url{https://deepai.org/dataset/citeseer}}: Cora and Citeseer are two classical single-view graph datasets. Each node represents a paper with features extracted from the bag-of-word of the paper. The network is a paper citation network.
    
    \item BBC Sport\footnote{\url{http://mlg.ucd.ie/datasets/bbc.html}}: BBC Sport is a widely-used traditional multi-view dataset. Each node is a document from one of five classes with two features extracted from the document. Two graphs are constructed from features in SDSNE~\cite{SDSNE}.
    \item 3Sources\footnote{\url{http://mlg.ucd.ie/datasets/3sources.html}}: 3Sources is a traditional multi-view dataset, containing 169 reports reported in different three sources. Each node represents a report with three features extracted from three sources respectively. The graphs are constructed from features in SDSNE~\cite{SDSNE}.
\end{itemize}
\begin{table}[]
\setlength\tabcolsep{2.4pt}
    \centering
    \caption{Clustering results on two single-view graph datasets (Cora and Citeseer) and two traditional multi-view datasets (BBC Sport and 3Sources). The values in the table are shown as percentage. The best and second-best results are shown in bold and underlined, and `$-$' represents the results not reported in the original papers.}
    \begin{tabular}{r|cccc|cccc}
    \toprule[1.5pt]
    \multirow{2}*{Methods / Datasets} & \multicolumn{4}{c|}{Cora} & \multicolumn{4}{c}{Citeseer} \\
         & NMI & ARI & ACC & F1 & NMI & ARI & ACC & F1 \\
    \midrule
    VGAE \cite{GAE}          (2016) & 40.8 & 34.7 & 59.2 & 45.6 & 16.3 & 10.1 & 39.2 & 27.8 \\
    ARVGE \cite{ARVGAE}      (2018) & 45.0 & 37.4 & 63.8 & 62.7 & 26.1 & 24.5 & 54.4 & 52.9 \\
    AGC \cite{AGC}           (2019) & 53.7 & 44.8 & 68.9 & 65.6 & 41.1 & 42.0 & 67.0 & 62.3 \\
    DAEGC \cite{DAEGC}       (2019) & 52.8 & \underline{49.6} & 70.4 & \underline{68.2} & 39.7 & 41.0 & 67.2 & \underline{63.6} \\
    GMM-VGAE \cite{GMM-VGAE} (2020) & 54.4 & $-$  & \underline{71.5} & 67.8 & 42.3 & $-$  & 67.4 & 63.2\\
    MAGCN \cite{MAGCN}       (2020) & \underline{55.3} & 47.6 & 71.0 & $-$  & 41.8 & 40.3 & \underline{69.8} & $-$ \\
    DNENC \cite{DNENC}       (2022) & 51.2 & 44.7 & 68.3 & $-$ & \underline{42.6} & \underline{44.9} & 69.2 & $-$ \\
    DAsNMF \cite{DAsNMF} (2024)& 36.8 & 24.4 & 47.2 & 37.6 & 16.4 & 12.1 & 40.0 & 31.7 \\
    \midrule
    \textbf{VGMGC (ours)} & \textbf{57.2} & \textbf{53.5} & \textbf{73.6} & \textbf{71.6} & \textbf{44.7} & \textbf{46.5} & \textbf{69.9} & \textbf{65.2} \\
    
    \midrule[1pt]
    
    Methods / Datasets & \multicolumn{4}{c|}{BBC Sport} & \multicolumn{4}{c}{3Sources}\\
    \midrule
    GMC \cite{GMC}        (2020) & 70.5 & 60.1 & 73.9 & 72.1 & 54.8 & 44.3 & 69.2 & 60.5 \\
    CGD \cite{CGD}        (2020) & 91.0 & 93.1 & 97.4 & 94.7 & 69.5 & 61.1 & 78.1 & 70.9 \\
    O2MAC \cite{o2multi}  (2020) & 89.1 & 90.6 & 96.4 & 96.5 & 72.7 & 75.5 & 65.0 & 66.9\\
    SDSNE \cite{SDSNE}    (2022) & \underline{94.8} & \underline{95.8} & \underline{98.5} & \underline{96.8} & \underline{84.8} & \underline{86.7} & \underline{93.5} & \underline{89.8}\\
    PGSC \cite{pgsc}      (2023) & 86.6 & 88.5 & 95.6 & 91.3 & 71.4 & 66.0 & 81.06 & 74.4\\
    TUMCR \cite{TUMCR}    (2024)& 89.0 & $-$ & 96.7 & $-$ & $-$ & $-$ & $-$ & $-$ \\
    SIWCL \cite{SIWCL}    (2024)& {91.9} & {93.0} & {97.6} & $-$ & $-$ & $-$ & $-$ & $-$ \\
    \midrule
    \textbf{VGMGC (ours)} & \textbf{96.1} & \textbf{96.9} & \textbf{98.9} & \textbf{99.0} & \textbf{85.9} & \textbf{87.6} & \textbf{94.1} & \textbf{92.7} \\
    
    \bottomrule[1.5pt]
    \end{tabular}
\label{tab:more_exp}
\end{table}
\subsubsection{Comparison Methods}
On four multi-view graph datasets, our comparison methods in Table~\ref{tab:overall_results} can be divided into three categories:
\begin{itemize}
    \item Single-view baselines: LINE \& GAE: LINE~\cite{LINE} and GAE~\cite{GAE} are two typical single-view clustering methods.
    \item Multi-view graph baselines: PMNE~\cite{PMNE}, SwMC~\cite{nie2017SwMC} and MNE~\cite{zhan2018MNE} are four classical multi-view clustering methods that try to generate final embedding for clustering. RMSC~\cite{xia14RMSC} is a robust spectral clustering model.
    \item SOTAs: O2MAC~\cite{o2multi} and MAGCN~\cite{MAGCN} are the methods that learn from both attribute features and structural information. In addition, graph filter-based MGC methods, such as MvAGC~\cite{lin2021graph} and MCGC~\cite{pan2021multi}, are included. COMPLETER~\cite{lin2021completer} and MVGRL~\cite{hassani2020mvgrl} are contrastive learning methods to learn a common representation for clustering. The latest methods, R$^2$FGC~\cite{R2FGC}, DuaLGR~\cite{dulgr}, BTGF~\cite{BTGF}, and DIAGC~\cite{DIAGC} are also included for comparison.
\end{itemize}
On other traditional two single-graph (Cora and Citeseer) and two multi-view datasets (BBC Sport and 3Scources), eight baselines dedicated to single-graph clustering and seven baselines for traditional multi-view clustering are benchmarked. Specifically:
\begin{itemize}
    \item Traditional single-graph clustering methods: VGAE~\cite{GAE}, ARVGE~\cite{ARVGAE}, and GMM-VGAE~\cite{GMM-VGAE} are three variational-based baselines; AGC~\cite{AGC} is a classical graph clustering baseline. DAEGC~\cite{DAEGC}, DNENC~\cite{DNENC}, MAGCN~\cite{MAGCN}, and DAsNMF~\cite{DAsNMF} are four latest clustering SOTAs conducted on single-view graph data.
    \item Traditional multi-view clustering methods: GMC~\cite{GMC}, CGD~\cite{CGD}, SDSNE~\cite{SDSNE} and PGSG~\cite{pgsc} are graph-based clustering methods, which are designed to learn a unified graph from multi-view data for clustering. O2MAC~\cite{o2multi} conducts GNN for multi-view clustering. TUMCR~\cite{TUMCR} and SIWCL~\cite{SIWCL} are two latest multi-view clustering methods. 
\end{itemize}

\subsubsection{Evaluations}
Following the work of the pioneers, we adopt four common metrics, \emph{i.e.}, normalized mutual information (NMI), adjusted rand index (ARI), accuracy (ACC), and F1-score (F1), as the metrics to evaluate the clustering performance of the proposed VGMGC. For comparison, we collect the best-reported performances from the literature. More implementation details can be found in Appendix~\ref{appImpleDetail}.

\subsection{Overall Results}
Table \ref{tab:overall_results} shows the clustering performance of the proposed VGMGC and comparison methods on four multi-graph datasets. Intuitively, VGMGC outperforms the best SOTA, \emph{i.e.}, MCGC, with ACC boosting 6.9\%, 2.5\%, on the two Amazon datasets respectively. Moreover, it shows an overwhelming advantage with respect to ACM and competitive results on DBLP. In all, with the help of the generated variational consensus graph, the proposed VGMGC achieves competitive performance compared with the existing MGC SOTAs. 

\subsection{Generalization to Other Datasets}
In this section, we explore how well VGMGC performed on the other two types of datasets, \emph{i.e.}, single-view graph datasets Cora and Citeseer, and traditional multi-view datasets 3Sources and BBC Sport. 

In Table~\ref{tab:more_exp}, it is clear that VGMGC outperforms the previous SOTA methods on both single-view graph datasets and traditional multi-view datasets. 
Specifically, we can see that the F1 score of VGMGC increases by 2.1\% on Cora and 2.9\% on 3Sources compared to the previous graph clustering and multi-view clustering baselines, GMM-VGAE and SDSNE, respectively,
which demonstrates that VGMGC can be well generalized to both single-view graph clustering and traditional multi-view clustering tasks.
\subsection{Ablation Study}\label{ablation}
\begin{table*}[]
    \centering
    \caption{The ablation study results on ACM and DBLP. The original results of VGMGC are shown in bold, and the third view $\mathcal{G}^{\it3}$ that does not exist in ACM is denoted by `*'. Values in parentheses compare the performance with the original VGMGC.}
    \label{tab:ablation}
    \begin{tabular}{l|cccc|cccc}
    \toprule[1.5pt]
         \multirow{2}*{Components} & \multicolumn{4}{c|}{ACM} & \multicolumn{4}{c}{DBLP}\\
          & NMI\% & ARI\% & ACC\% & F1\% & NMI\% & ARI\% & ACC\% & F1\%\\
        \midrule
         w/o $\mathcal{L}_{c}$ & 75.0 (-1.3) & 80.6 (-1.3) & 93.1 (-0.5) & 93.1 (-0.5) & 77.9 (-0.4) & 83.2 (-0.5) & 93.0 (-0.2) & 92.4 (-0.3)\\
         w/o $\mathcal{L}_r$  & 67.6 (-8.7) & 70.4 (-11.5) & 88.7 (-4.9) & 88.4 (-5.2) & 76.2 (-2.1) & 81.4 (-2.3) & 92.1 (-1.1) & 91.3 (-1.4) \\
         w/o $\mathbf{\hat{S}}$ \& $\mathcal{L}_{E}$  & 44.4 (-31.9) & 42.4 (-39.5) & 70.9 (-22.7) & 71.4 (-22.2) & 20.7 (-57.6) & 17.8 (-65.9) & 47.1 (-46.1) & 44.8 (-47.9)\\
         \midrule
        VGMGC-single ($\mathcal{G}^{\it1}$) & 73.1(-3.2)  & 79.6(-1.3) & 92.8(-0.8) & 92.8(-0.8)  & 47.7(-30.6) & 52.5(-31.2) & 79.0(-14.2) & 78.6(-14.1) \\
        VGMGC-single ($\mathcal{G}^{\it2}$) & 59.3(-17.0)  & 61.2(-20.7) & 84.4(-9.2) & 84.0(-9.6)  & 75.8(-2.5) & 80.9(-2.8) & 91.9(-1.3) & 91.0(-1.7) \\
        VGMGC-single ($\mathcal{G}^{\it3}$) & *     & *    & *    & *     & 15.6(-62.7) & 11.1(-72.6) & 45.9(-47.3) & 45.3(-47.4) \\
        \midrule 
         \textbf{Original}   & \textbf{76.3} & \textbf{81.9} & \textbf{93.6} & \textbf{93.6} & \textbf{78.3} & \textbf{83.7} & \textbf{93.2} & \textbf{92.7}\\
         \bottomrule[1.5pt]
    \end{tabular}
\end{table*}

We present the ablation study of VGMGC on ACM and DBLP datasets in this subsection, including the effect of each component and the performance on each single view.
\subsubsection{Effect of Each Component}
As shown in the upper part of Table~\ref{tab:ablation}, without any of the components, the performance of VGMGC is negatively affected to varying degrees. Specifically, the generated consensus graph $\mathbf{\hat{S}}$ plays an important role. Without the consensus graph and its guidance (see w/o $\mathbf{\hat{S}}$ \& $\mathcal{L}_{E}$), the dramatic drop in performance demonstrates that the variational consensus graph $\mathbf{\hat{S}}$ can mine enough global topological information.$\mathcal{L}_r$ also makes a considerable contribution to the overall performance (see w/o $\mathcal{L}_r$), as the $\mathbf{Z}^v$ of specific view can preserve useful information from $\mathbf{X}^v$. 
Furthermore, this means the latent representations $\mathbf{\overline{Z}}$ have included enough clustering information. Thus, the clustering loss $\mathcal{L}_{c}$ is helpful, while it seems to be faint (see w/o $\mathcal{L}_c$).

More specifically, $\mathcal{L}_{c}$ is the clustering loss (Eq.~\eqref{eq:KLloss}) conducting on final embeddings $\overline{\bf{Z}}$. The lower effect of $\mathcal{L}_{c}$ shown in Table~\ref{tab:ablation} implies that the better global clustering information is already obtained by the embedding $\overline{\bf{Z}}$, potentially demonstrating that the components of variational graph generator and graph encoder could effectively capture the global clustering information without clustering loss.
\begin{figure}[t]
    \includegraphics[width=0.95\linewidth]{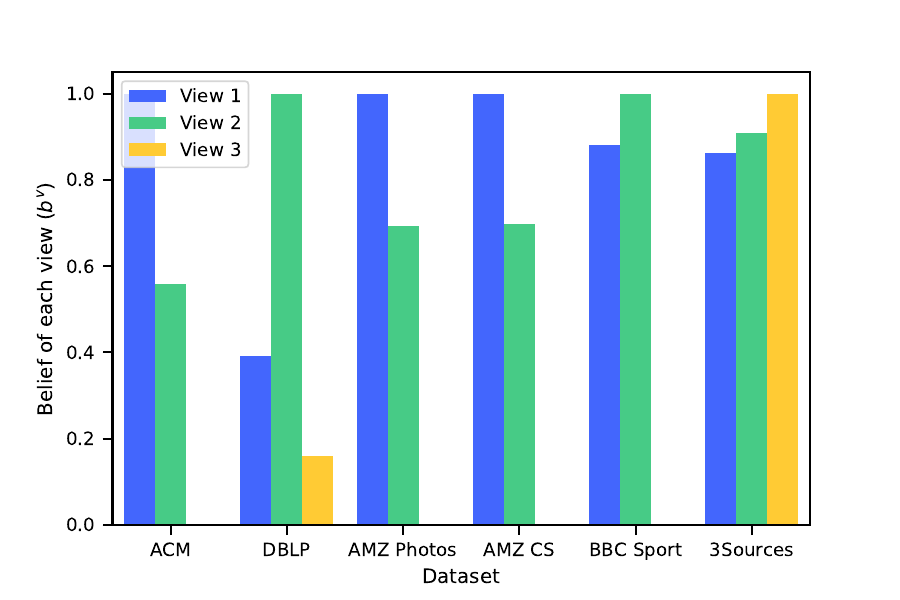}
    \caption{The learned $b^v$ on six multi-view datasets, which shows the task relevance of each view. The parameter $\rho$ is fixed to $1$ in this figure. `AMZ' means Amazon and `CS' denotes Computers.}
    \label{fig:b_barplot}
\end{figure}
\subsubsection{Performance on Each View}\label{Seceachview}
We run VGMGC on each single view to test how much effect VGMGC has gained from these multiple views helping with each other. The last three rows of Table~\ref{tab:ablation} show the results (see VGMGC-single), and Figure~\ref{fig:b_barplot} depicts the task relevance of each view ($b^v$) learned by VGMGC on six multi-view datasets. It can be seen that the clustering results of VGMGC (original) outperform all other single views, which indicates that VGMGC can take advantage of all views to obtain better clustering results. For example, on DBLP, even though the first graph and third graph perform badly as View 1 and View 3 shown in Figure~\ref{fig:b_barplot}, the NMI of VGMGC still increases by 2.5\% compared to the best single-view results (VGMGC-single ($\mathcal{G}^{\it2}$)). One possible reason for this is that the noisy part of a feature in one view is corrected by the consideration of other views in VGMGC. This can be partially demonstrated via the visualization of the clustering results of DBLP shown in Figure~\ref{fig:visview}. Comparing the visualization of $\overline{\textbf{Z}}$ and $\textbf{Z}^{\it2}$, it can be observed that the error samples in the best view ($\mathbf{Z}^{\it2}$ in Figure~\ref{fig:b_barplot}) are corrected after the combination of all views ($\mathbf{\overline{Z}}$ in Figure~\ref{fig:b_barplot}).
\begin{figure}[t]
    \centering
    \subfloat{
        \includegraphics[width=1.63in]{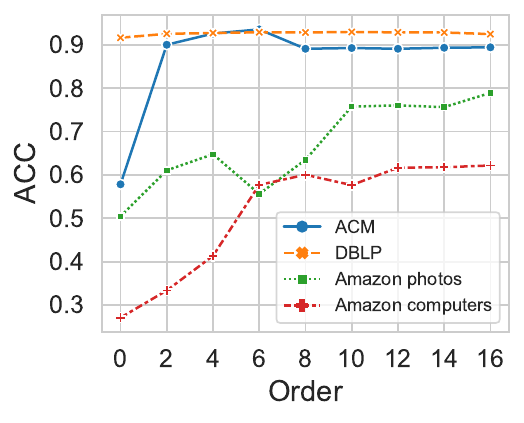}}
    \subfloat{
        \includegraphics[width=1.58in]{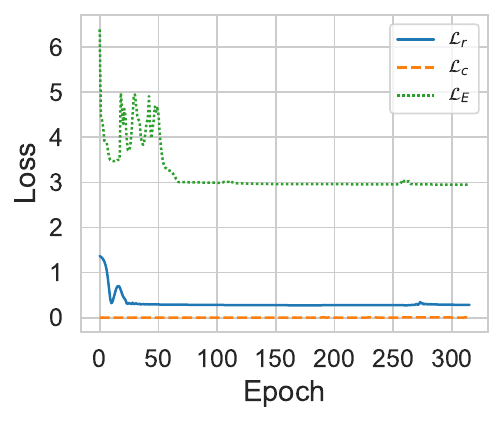}}
    \caption{The performance of VGMGC with different orders (left), and the training process on ACM (right).}
    \label{fig:smooth_loss}
\end{figure}
\subsection{Analysis and Visualization}
\subsubsection{Oversmoothing Analysis}
The left of Figure~\ref{fig:smooth_loss} shows the performance of VGMGC with different orders. Overall, the ACC increases with the increase of orders, which suggests that VGMGC can mitigate oversmoothing. This effect benefits from the elaborately designed massage-passing scheme in Section~\ref{secGraphEnc}. Notably, the plot on DBLP is steady even when the aggregation order is $0$, which might be due to that VGMGC has learned sufficient clustering-related information from original features directly before message-passing.

\subsubsection{Convergence Analysis}
The right of Figure~\ref{fig:smooth_loss} shows the trend of losses with training epochs. We can see that all losses converge after training around 100 epochs. Specifically, the value of $\mathcal{L}_E$ is large, and $\mathcal{L}_c$ is relatively small.
\begin{figure*}[t]
    \centering
    \subfloat[NMI on ACM.]{
        \includegraphics[width=1.6in, height=1.5in]{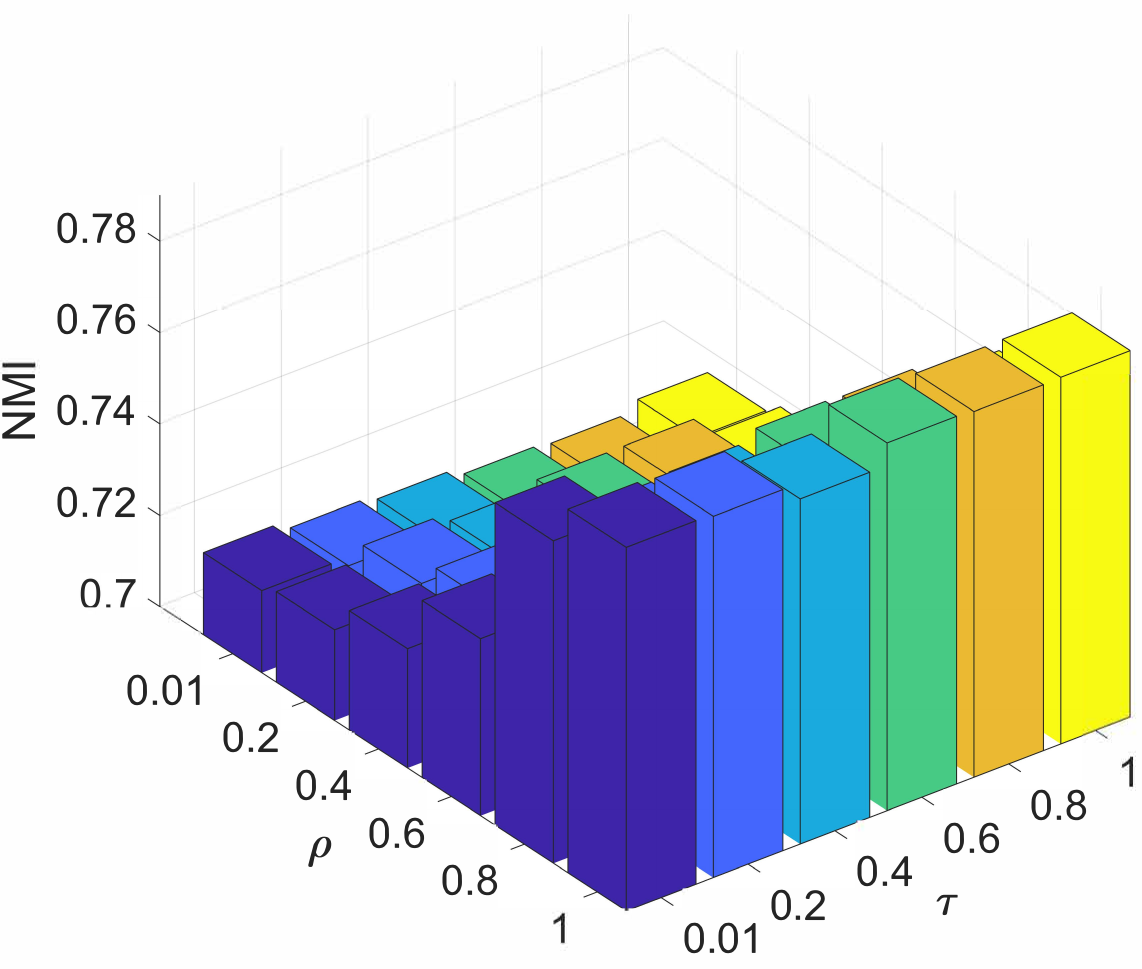}}
    \subfloat[NMI on DBLP.]{
        \includegraphics[width=1.6in, height=1.5in]{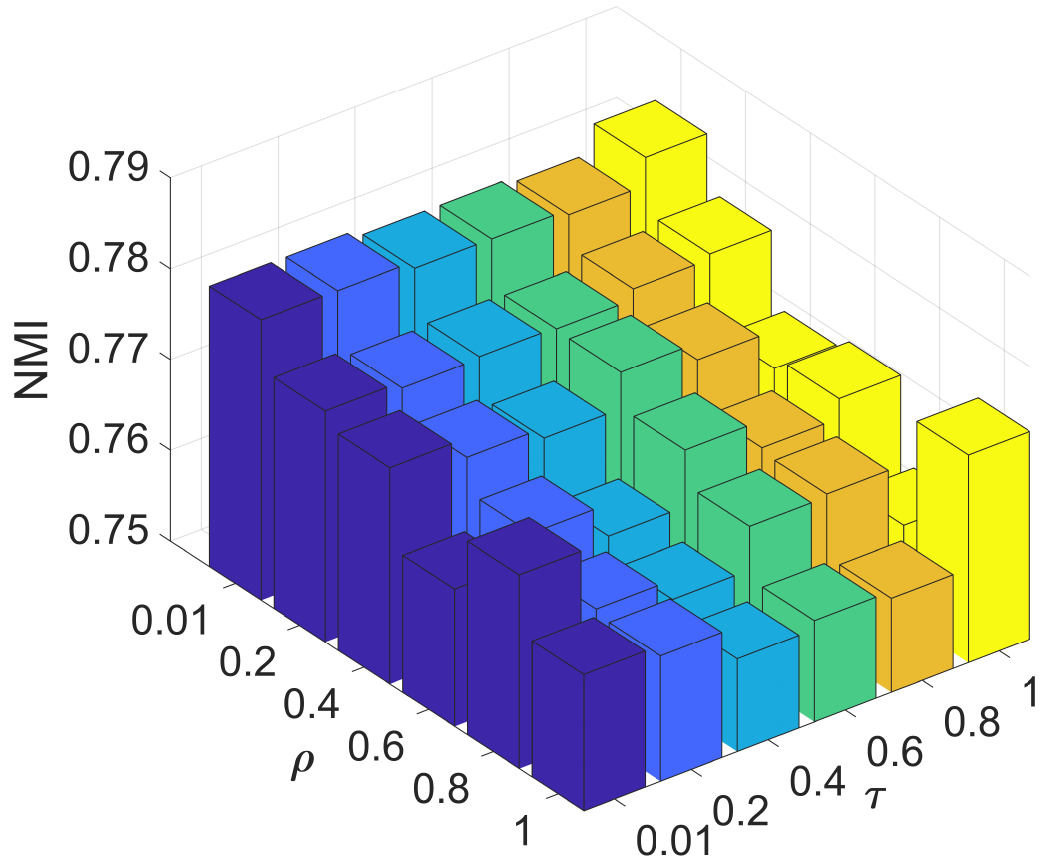}}
    \subfloat[ACC on ACM.]{
        \includegraphics[width=1.6in, height=1.5in]{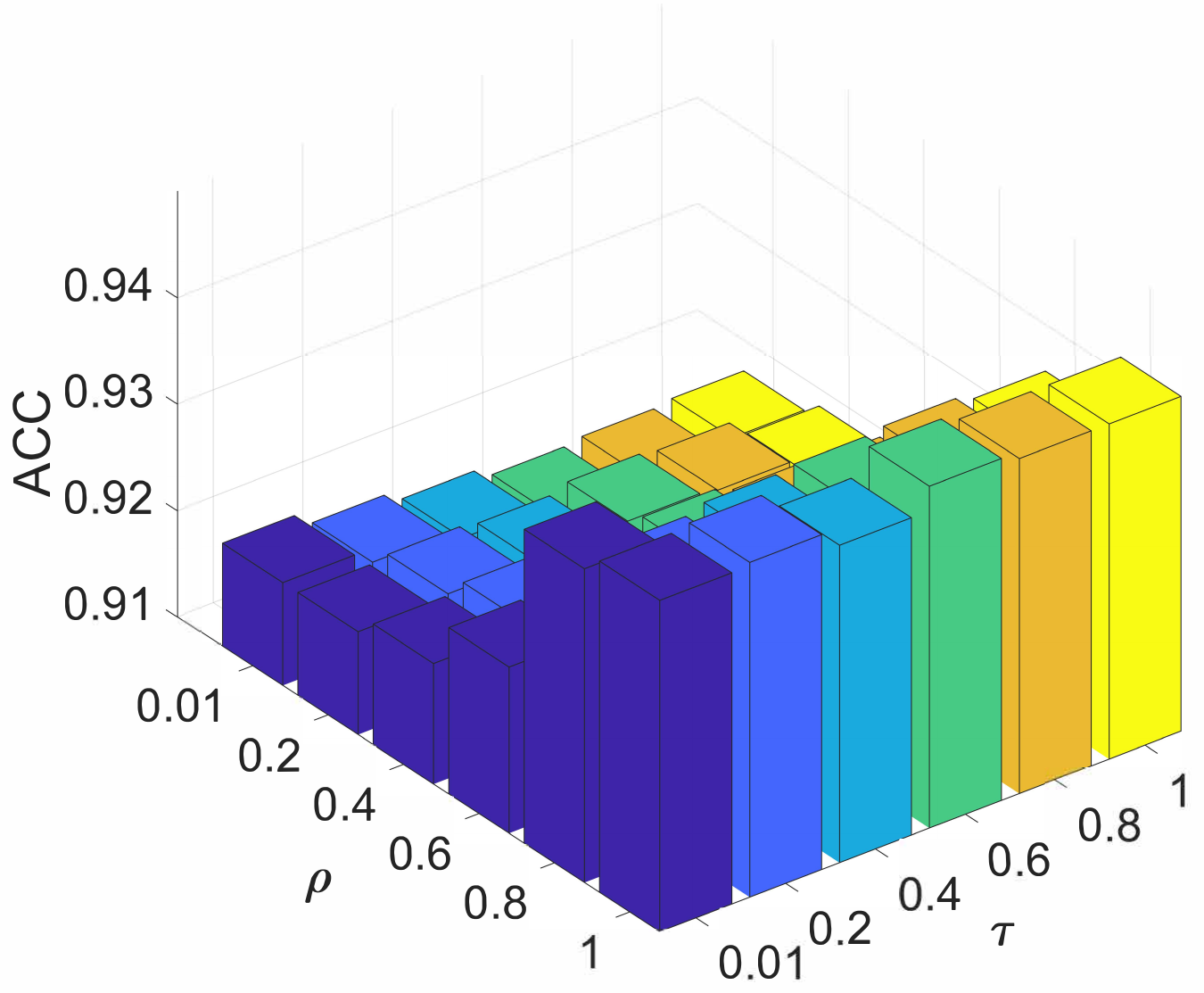}}
    \subfloat[ACC on DBLP.]{
        \includegraphics[width=1.6in, height=1.5in]{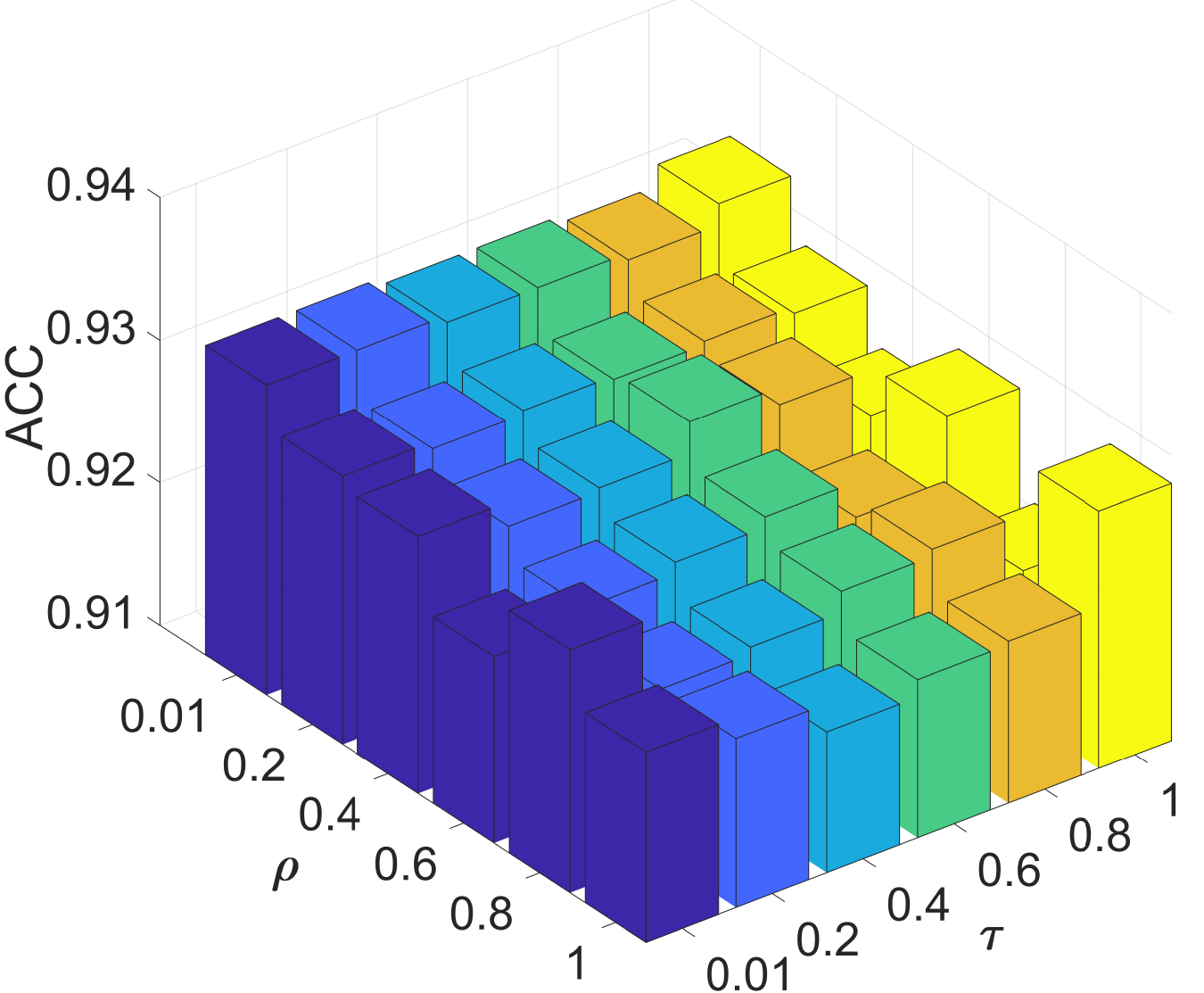}}
    \caption{Sensitive analysis with $\rho$ and $\tau$ on ACM and DBLP.}
    \label{fig:sensitive}
\end{figure*}
\begin{figure*}[t]
    \centering
        \includegraphics[width=\linewidth]{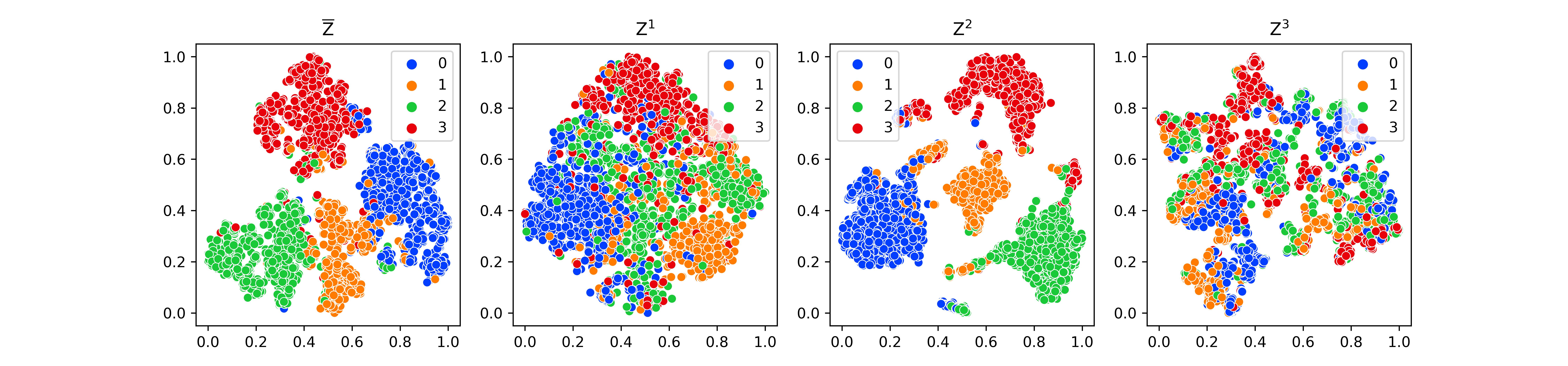}
    \caption{The visualization of $t$-SNE on each view's latent representation of DBLP dataset. $\mathbf{\overline{Z}}$ denotes the multi-view fused latent representation, and the rest denotes other views' representation of the dataset.}
    \label{fig:visview}
\end{figure*}
\subsubsection{Parameter Sensitive Analysis}
Figure~\ref{fig:sensitive} depicts how much the parameter $\rho$ and $\tau$ influence VGMGC.
In general, $\rho$ is more influential than $\tau$, and the influences of both $\rho$ and $\tau$ are minor to VGMGC. The stable performance regarding the two parameters is because they can only sharpen or smooth, rather than determine, the positive or negative influences from multiple views, and these influences are also adaptively adjusted by VGMGC itself. 
More specifically, Eq.~\eqref{eqBelief} shows that the belief of each view is determined by the adaptively adjusted scores instead of hyper-parameter $\rho$. Similarly, Eq.~\eqref{eqRepara} indicates that the temperature parameter $\tau$ cannot determine the final distribution of the generated graph. 
It is also notable that the ACC performs stable on the two datasets with the variation of ACC within 3\%. Looking into the results of NMI, VGMGC has a more stable performance on DBLP. Specifically, it can be seen that NMI trends to increase on ACM but decrease on DBLP as the hyper-parameter $\rho$ increases. We could speculate that the information from different views may be harmful on ACM but more compatible and helpful on DBLP. This is because the $\rho$ influences how much the bad view affects our model, as discussed in Section~\ref{SecMVFusion}. VGMGC, on the other hand, behaves relatively stable for different values of $\tau$ which impacts the distribution of $s_{ij}$.

Furthermore, we explore a wider range of $\tau$ to observe its influence on the distribution of the generated graph, which is represented by the mean and variance of $s_{ij}$. According to Table~\ref{tab:tau},  as $\tau$ increases, the variance of $s_{ij}$ decreases and its mean value gradually approaches 50\%.  This occurs because increasing $\tau$ introduces more noise from a uniform distribution ($Uniform(0,1)$) into $s_{ij}$, which consistently aligns with the theory.

\begin{table}[t]
    \centering
    \caption{Experiments of larger $\tau$ on ACM. The last two columns denote the mean and variance of $s_{ij}$ computed from training VGMGC with 200 epochs.}
    \begin{tabular}{l|llll}
    \toprule[1.5pt]
    Temperature & NMI\% & ACC\% & Mean (1e-2) & Variance (1e-2)\\
    \midrule
    $\tau=0.1$  &  77.4 & 93.8 &	50.46	& 22.5\cr
    $\tau=1$    &  77.5 & 93.9 &	50.297	& 8.337 \cr
    $\tau=2$    &  77.7 & 94.0 &	50.192	& 3.542 \cr
    $\boldsymbol{\tau=5}$    &  \textbf{77.7} & \textbf{94.1} & 50.091	& 0.758 \cr
    $\tau=10$   &  77.7 & 94.0 &	50.046	& 0.201 \cr
    $\tau=50$   &  77.6 & 94.0 &	50.009	& 8.217 (1e-5) \cr
    $\tau=100$  &  77.2 & 93.9 &	50.0004	& 2.056 (1e-5) \cr
    \bottomrule[1.5pt]
    \end{tabular}
\label{tab:tau}
\end{table}
\subsubsection{Visualization of Each View}\label{secVisualize}
Figure~\ref{fig:visview} visualizes the embedding learned by VGMGC for each view, providing a more intuitive understanding of the outstanding clustering performance of the proposed VGMGC.
Obviously, the third view $\mathbf{Z}^{\it3}$ in DBLP performs poorly, whereas $\mathbf{Z}^{\it2}$ in DBLP performs relatively well. It is worth noting that the generated variational fused view $\mathbf{\overline{Z}}$ can differentiate some ambiguous points in the second view $\mathbf{Z}^{\it2}$. The fused view manages to have an impressive performance by effectively mining enough view-specific and view-common information from all views. This finding corroborates the results presented in Section~\ref{Seceachview}.

\section{Conclusions}
\label{secCon}
In this paper, we proposed a multi-view graph clustering method to better mine and embed the consensus topological information of all graphs while reserving useful feature information and extracting some task-relevant information, which is called variational graph generator for multi-view graph clustering (VGMGC). We also theoretically proved its effectiveness from information theories. Our experiments showed that VGMGC outperforms previous MGC works on eight datasets.
In the future, we expect our work could inspire more research, such as graph generation and multi-graph learning.
A limitation of the proposed method is that it is limited to the clustering task. For other tasks, it needs to change the score function to compute the belief of each graph.

\section*{Acknowledgements} 
This work is supported in part by National Natural Science Foundation of China (No. 62476052), Sichuan Science and Technology Program (No. 2024NSFSC1473), and Shenzhen Science and Technology Program (No. JCYJ20230807115959041). Philip S. Yu is supported in part by NSF under grant III-2106758. Lifang He is partially supported by the NSF grants (MRI-2215789, IIS-1909879, IIS-2319451), NIH grant under R21EY034179, and Lehigh’s grants under Accelerator and CORE.

\begin{small}
\bibliographystyle{IEEEtranN}
\bibliography{main_arxiv}
\end{small}

\appendix

\subsection{Overall Notations}
\label{appNotation}
The explanations of all notations are shown in Table~\ref{tab:my_label}.
\begin{table}[!t]
    \centering
    \caption{Overall notations.}
    \begin{tabular}{cl}
    \toprule
        $\mathbf{A}^v$ & $\mathbf{A}^v \in \mathbb{R}^{n \times n}$, $v$-th view's adjacent matrix with self-loop.\\
        \midrule
        $\mathbf{D}^v$ & $\mathbf{D}^v_{ii}=\sum_j a^v_{ij}$, degree matrix of $\mathbf{A}^v$.\\
        \midrule
        $\mathbf{\widetilde{A}}^v$ & $\mathbf{\widetilde{A}}^v=(\mathbf{D}^v)^{-1}\mathbf{A}^v$, normalized adjacent matrix.\\
        \midrule
        $\mathbf{\breve{A}}^v$ & Reconstructed graph in generative process (Eq.~\eqref{eqGen}).\\
        \midrule
        $\mathbf{X}^v$ & $\mathbf{X}^v \in \mathbb{R}^{n\times d_v}$, $v$-th view's attributed features.\\
        \midrule
        $\mathbf{\breve{X}}^v$ & Reconstructed features to calculate $\mathcal{L}_r$.\\
        \midrule
        $\mathbf{\overline{X}}$ & $\mathbf{\overline{X}} \in \mathbb{R}^{n \times d'}$, global features from all views' features.\\
        \midrule
        $\mathcal{G}^v$ &  $\mathcal{G}^v=(\mathbf{X}^v,\mathbf{A}^v)$, $v$-th view's graph.\\
        \midrule
        $\mathbf{Z}^v$ & $\mathbf{Z}^v \in \mathbb{R}^{n \times D}$, $v$-th view's graph embedding.\\
        \midrule
        $\mathbf{\overline{Z}}$ &$\mathbf{\overline{Z}} \in \mathbb{R}^{n \times (V \cdot D)}$, final global embedding from $\{\mathbf{Z}^v\}^V_{v=1}$.\\
        \midrule
        $\boldsymbol{\alpha}$ & $\boldsymbol{\alpha} \in \mathbb{R}^{n\times n}$ are the neurons used to obtain $\mathbf{S}$.\\
        \midrule
        $\mathbf{S}$ & $\mathbf{S}=\{s_{ij}\mid s_{ij} \thicksim BinConcrete(\alpha_{ij}, \tau)\}$.\\
        \midrule
        $\mathbf{\hat{S}}$ & $\mathbf{\hat{S}}\in \mathbb{R}^{n\times n}$, the variational consensus graph.\\
        \midrule
        $b^v$ & The belief of $v$-th graph, calculated from Eq.~\eqref{eqBelief}.\\
        \midrule
        $\beta_{ij}$ & $\beta_{ij}\in (0,1]$, the parameter of Bernoulli distribution.\\
        \midrule
        $U$ & Sampled from Uniform distribution, $U \thicksim Uniform(0,1)$.\\
        \midrule
        $\phi'$ & The learnable parameters in $q_{\phi'}(\mathbf{\hat{S}}\mid\mathbf{\overline{X}})$.\\
        \midrule
        $\eta'$ & The learnable parameters in $q_{\eta'}(\boldsymbol{\alpha} \mid \mathbf{\overline{X}})$.\\
        \midrule
        $\eta^v$ & The learnable parameters in $q_{\eta^v}(\mathbf{X}^v \mid \mathbf{Z}^v)$.\\
        \midrule
        $\phi^v$ & The learnable parameters of $f_v$. \\
        \midrule
        $\rho$ & $\rho \geq 0$ is a hyper-parameter for belief $b^v$ in Eq.~\eqref{eqBelief}.\\
        \midrule
        $\tau$ & $\tau > 0$ is a hyper-parameter of temperature in Eq.~\eqref{eqRepara}.\\
        \midrule
        $order$ & A hyper-parameter of aggregation orders in Eq.~\eqref{eqMessagePassing}.\\
        \midrule
        $\mathcal{L}_c$  & Task related loss for optimizing clustering results.\\
        \midrule
        $\mathcal{L}_r$  & Reconstruction loss for optimizing mutual information.\\
        \midrule
        $\mathcal{L}_{E}$ & ELBO for optimizing variational graph generator.\\
        \midrule
        $\gamma_c$  & A hyper-parameter to trade off the numerical value of $\mathcal{L}_c$.\\
        \midrule
        $\gamma_E$  & A hyper-parameter to trade off the numerical value of $\mathcal{L}_E$.\\
    \bottomrule
    \end{tabular}
    \label{tab:my_label}
\end{table}

\subsection{Implementation Details}
\label{appImpleDetail}
\subsubsection{Experimental Environments}
The experiments are conducted on a CentOS machine with a NVIDIA Tesla V100 SXM2 GPU and Cascade LakeP82 v6@2.4GHz CPU. CUDA version is 10.2 and PyTorch version is 1.11.0. 
\subsubsection{Details of Global and Specific Graph Encoder}
Similar to the MLP $f'$, $f_v$ also has three layers, the dimensions of $hidden$ and $output$ layers are chosen to be $512$. The parameters $\phi^v$ are initialized with Kaiming, and \emph{ReLu} is selected as the activation function.

\subsubsection{Details of Inference Process}
The variational consensus graph $\mathbf{\hat{S}}$ is generated through a $f'$ from global features $\mathbf{\overline{X}}$. The instantiation of $f'$ is an MLP, playing a role of encoder to extract the information from $\mathbf{\overline{X}}$. Specifically, $f'$ consists of three layers, \emph{i.e.}, $input \rightarrow hidden \rightarrow output$, and each layer follows a Dropout operation. In our model, we set the dimensions of both $hidden$ and $output$ layers as $512$, drop rate of Dropout operation as $0.1$ and all parameters of $\phi'$ are initialized with Xavier.

\subsection{Details of Loss Functions}
\label{appLossFunction}
\subsubsection{Evidence Lower Bound Loss}
In the generative process, the decoder of $p_\xi^v(\mathbf{A}^v|\mathbf{\hat{S}},\mathbf{X}^v)$ shares structure and parameters with $f_v$ and $\phi^v$. Finally, $\breve{\mathbf{A}}^v$ can be reconstructed through $\sigma(\mathbf{Z}^v(\mathbf{Z}^v)^T)$, where $\sigma(\cdot)$ denotes \emph{Sigmoid} function. Thus $\mathcal{L}_{E}$ can be written as:

\begin{equation}
    \mathcal{L}_{E} 
    = -\sum^V_{v=1} H(\mathbf{A}^v, \mathbf{\breve{A}}^v) + H(\mathbf{\hat{S}}) - \sum_{ij}\log{\frac{1}{\beta_{ij}}}.
\end{equation}
\subsubsection{Lower Bound of Mutual Information Loss}
For the reconstruction of $\mathbf{X}^v$, $\mathbf{Z}^v$ is fed into an MLP decoder, whose structure is the opposite of $f_v$. To reconstruct the global features $\mathbf{\overline{X}}$ to $\mathbf{\breve{\overline{X}}}$, the obtained $\mathbf{Q}$ is fed into an MLP decoder opposite of $f'$. Subsequently, the reconstruction losses $\mathcal{L}^v_r$ and $\mathcal{L}'_r$ for maximizing the lower bound of mutual information can be written as follows:
\begin{equation}
\begin{aligned}
    \mathcal{L}^v_r 
    &= \sum^V_{v=1} H(\mathbf{X}^v, \mathbf{\breve{X}}^v) \\
    &= -\sum^V_{v=1} \sum_{i}
    \left[\mathbf{x}_{i}^v \log \mathbf{\breve{x}}_{i}^v + (1-\mathbf{x}_{i}^v) \log (1-\mathbf{\breve{x}}_{i}^v) \right],
\end{aligned}
\end{equation}

\begin{equation}
\begin{aligned}
    \mathcal{L}'_r 
    &= H(\mathbf{\overline{X}},\mathbf{\breve{\overline{X}}})\\
    &= -\sum_{i} \left[\mathbf{\overline{x}}_{i} \log \mathbf{\breve{\overline{x}}}_{i} + (1-\mathbf{\overline{x}}_{i}) \log (1-\mathbf{\breve{\overline{x}}}_{i})\right].    
\end{aligned}
\end{equation}
Finally, we can obtain our overall reconstruction loss $\mathcal{L}_r$:
\begin{equation}
\mathcal{L}_r = \sum_v\mathcal{L}^v_r + \mathcal{L}'_r
\end{equation}

\subsubsection{Clustering Loss}
Clustering loss is specialized for clustering task~\cite{MAGCN,DEC,xu2021deep}, which encourages the assignment distribution of samples in a same cluster being more similar. Concretely, let $Q^v=\{q^v_{ij}\}$ be a soft assignment of $i$-th node to $j$-the cluster in $v$-th view, $\boldsymbol{\mu}^v \in \mathbb{R}^{c \times D}$ be $c$ centroids of $c$ clusters, $q^v_{ij}$ is calculated by Student's $t$-distribution~\cite{tsne}:
\begin{equation}
    q^v_{ij} = \frac{(1+\Vert \mathbf{z}^v_i-\boldsymbol{\mu}^v_j\Vert^2)^{-1}}{\sum_j(1+\Vert\mathbf{z}^v_i-\boldsymbol{\mu}^v_j\Vert^2)^{-1}}.
\end{equation}
By sharpening the soft assignment $q^v_{ij}$, we can obtain its target distribution $P^v=\{p^v_{ij}\}$:
\begin{equation}
    p^v_{ij}=\frac{(q^v_{ij})^2 / \sum_iq^v_{ij}}{\sum_j ((q^v_{ij})^2 / \sum_iq^v_{ij})}.
\end{equation}
The clustering loss encourages soft assignment distribution $Q^v$ to fit target distribution $P^v$ by KL divergence, \emph{i.e.}, $KL(P^v \Vert Q^v)$. In our multi-view clustering task, we assume that every view's soft distribution could fit the most correct target distribution, \emph{i.e.}, the target distribution of global representation $\mathbf{\overline{Z}}$. Using $\overline{P}$ and $\overline{Q}$ denote the target and soft distribution of $\mathbf{\overline{Z}}$, the multi-view clustering loss could be improved as:
\begin{equation}
    \mathcal{L}_{c} = \sum^V_{v=1} KL(\overline{P} \Vert Q^v) + KL(\overline{P} \Vert \overline{Q}).
\end{equation}


 




\end{document}